\documentclass[11pt]{article}

\usepackage{fancybox}
\usepackage{algorithm, algpseudocode}
\usepackage{mathtools}
\mathtoolsset{showonlyrefs=true}

\usepackage{enumitem}
\usepackage{color}
\usepackage{url}
\usepackage[margin=1in]{geometry}
\usepackage{amsmath,amssymb}
\usepackage{verbatim}
\usepackage{graphicx,setspace}
\usepackage{amsthm}
\usepackage{tabularx}
\usepackage{cite}
\allowdisplaybreaks

\newcommand*{\KeepStyleUnderBrace}[1]{
  \mathop{%
    \mathchoice
    {\underbrace{\displaystyle#1}}%
    {\underbrace{\textstyle#1}}%
    {\underbrace{\scriptstyle#1}}%
    {\underbrace{\scriptscriptstyle#1}}%
  }\limits
}

\usepackage[utf8]{inputenc}
\usepackage{hyperref}
\hypersetup{
    colorlinks=true,
    citecolor = blue,
    linkcolor=blue,
    filecolor=magenta,           
    urlcolor=cyan,
}

\usepackage{stmaryrd}
\usepackage{graphicx}
\usepackage{multirow}
\usepackage{booktabs}

\usepackage{amsmath,amssymb}
\usepackage{amsthm}
\usepackage{bm}

\def\nmB{\bm{\widetilde B}}

\def\nmSig{\bm{\widetilde \Sigma}}

\def\nL{\mathcal{\widetilde L}}

\def\nR{\mathcal{\widetilde R}}

\def\tA{\mathcal{A}}
\def\tB{\mathcal{B}}

\def\tE{\mathcal{E}}

\def\tT{\mathcal{T}}

\def\entry#1{\llbracket #1 \rrbracket}
\def\est{\bm{\widehat u}_1} 
\def\estv{\widehat \lambda_1}
\def\esti{\bm{\widehat u}_i}
\def\estvi{\widehat \lambda_i}
\def\true{\bm{u}_1} 
\def\truev{\lambda_1}
\def\ls#1{\mathcal{#1}} 
\def\ntensor{\mathcal{\widetilde T}} 
\def\nHOSVD{\mathcal{\widetilde T}_{(12)(3\ldots k)}} 
\def\tensor#1{\mathcal{#1}} 
\def\entry#1{\llbracket #1 \rrbracket}

\def\eqdef{\stackrel{\text{\tiny \rm def}}{=}} 
\def\Vec{\operatorname{Vec}}
\def\Spanspace{\operatorname{Span}}
\def\Loss{\operatorname{Loss}}
\def\bbR{\mathbb{R}}

\newcommand{\matr}[1]{\bm{#1}}
\newcommand{\norm}[1]{\left\lVert#1\right\rVert_\sigma}
\newcommand{\vectornorm}[1]{\left\lVert#1\right\rVert_2}
\newcommand{\Fnorm}[1]{\left\lVert#1\right\rVert_F}

\newcommand{\normSize}[2]{#1\lVert#2#1\rVert_\sigma}
\newcommand{\FnormSize}[2]{#1\lVert#2#1\rVert_F} 

\DeclareMathOperator*{\argmax}{arg\,max}
\DeclareMathOperator*{\maximize}{maximize}

\newcommand{\suppfref}[1]{{Supplementary Figure~\ref{#1}}}
\newcommand{\supptref}[1]{{Supplementary Table~\ref{#1}}}
\def\sref#1{Section~\ref{#1}}

\usepackage{tikz}
\usetikzlibrary{decorations.pathreplacing,calc}
\newcommand{\tikzmark}[1]{\tikz[overlay,remember picture] \node (#1) {};}

\newcommand*{\SpaceReservedForComments}{2.9cm}%
\newcommand*{\HorizontalOffset}{-0.5em}%
\newcommand*{\VerticalOffset}{.8ex}%
\newcommand*{\AddNote}[4][]{%
    \begin{tikzpicture}[overlay, remember picture]
        \draw [decoration={brace,amplitude=0.2em},decorate,thick,red, #1]
            ($(#3)+(\HorizontalOffset,-\VerticalOffset)$) --  ($(#2)+(\HorizontalOffset,\VerticalOffset)$)
            node [align=left, text width=\SpaceReservedForComments-1.0em, pos=0.5, anchor=east] {#4};
    \end{tikzpicture}
    }

\makeatletter
    \algrenewcommand\alglinenumber[1]{\tikzmark{\arabic{ALG@line}}\tiny#1:}
\makeatother

\theoremstyle{plain}
\newtheorem{thm}{Theorem}[section]
\newtheorem{lem}{Lemma}[section]
\newtheorem{cor}{Corollary}[section]
\newtheorem{prop}{Proposition}[section]

\theoremstyle{definition}
\newtheorem{defn}{Definition}[section]
\newtheorem{assume}{Assumption}[section]

\theoremstyle{remark}
\newtheorem{rmk}{Remark}[section]

\usepackage{algorithm}
\algnewcommand\algorithmicinput{\textbf{Input:}}
\algnewcommand\algorithmicoutput{\textbf{Output:}}
\algnewcommand\INPUT{\item[\algorithmicinput]}
\algnewcommand\OUTPUT{\item[\algorithmicoutput]}

\begin{document}

\begin{center}
\begin{spacing}{1.5}
\textbf{\Large  Tensor Decompositions via Two-Mode Higher-Order~SVD~(HOSVD)}
\end{spacing}

\vspace{5mm}
Miaoyan Wang$^1$ and Yun S. Song$^{1,2,3}$

\vspace{5mm}
$^1$Department of Mathematics, University of Pennsylvania\\
$^2$Department of Statistics, University of California, Berkeley\\
$^3$Computer Science Division, University of California, Berkeley 
\end{center}
\vspace{5mm}

\begin{abstract}

Tensor decompositions have rich applications in statistics and machine learning, and developing efficient, accurate algorithms for the problem has received much attention recently.
Here, we present a new method built on Kruskal's uniqueness theorem to decompose symmetric, nearly orthogonally decomposable tensors.
Unlike the classical higher-order singular value decomposition  which unfolds a tensor along a single mode, we consider unfoldings along two modes and use rank-1 constraints to characterize the underlying components. 
This tensor decomposition method provably handles a greater level of noise compared to previous methods and achieves a high estimation accuracy.  Numerical results demonstrate that our algorithm is robust to various noise distributions and that it performs especially favorably as the order increases.

\end{abstract}

\section{Introduction}

Tensor decompositions have recently drawn increased attention in statistical and machine learning applications. For example, tensor decompositions provide powerful tools to estimate parameters in various latent variable models via the method of moments \cite{anandkumar2012spectral, anandkumartensor, anandkumar2014tensor} and lead to the development of efficient denoising techniques in independent component analysis \cite{comon1994independent, comon2010handbook, vempala2015max}. 

Decomposition of higher-order tensors is fraught with challenges, since in general most computational problems regarding tensors are NP-hard \cite{hillar2013most}.  However, the tensors considered in the aforementioned applications possess special structures that facilitate computation.
A class of such tensors which we study in this paper are nearly orthogonal decomposable tensors of the following form: 
\begin{equation}\label{eq:problem}
\ntensor=\sum_{i=1}^r \lambda_i \matr{u}_i^{\otimes k}+\text{noise}, 
\end{equation}
where $\{\matr{u}_i\}_{i=1}^r$ is a set of orthonormal vectors in $\bbR^d$, $\lambda_i$ are non-zero scalars in $\bbR$, and the noise is assumed to be symmetric and small (bounded by a positive constant $\varepsilon$ under some natural measurement). Our goal is to estimate the factor pairs $\{(\matr{u}_i,\lambda_i)\in\bbR^d \times \bbR\}_{i=1}^r$ from the noisy observation $\ntensor$. 
Often the number $r$  of factors is also unknown and it needs to be determined from $\ntensor$; we consider this general case at the end of the paper.
In contrast to existing works that have mostly focused on order-3 tensors (i.e., $k=3$), we here present a general framework applicable to an arbitrary order $k\geq 3$.  Unless stated otherwise, we use the term ``tensor'' to mean a tensor of order $\geq 3$. 

Current tensor decomposition methods vary in their objective function, tolerance of the noise level $\varepsilon$, target range of $r$, and runtime. The robust tensor power method (TPM) \cite{anandkumar2014tensor}, which computes the fixed point of a certain iterative function,  is a natural generalization of the matrix power method to the $k=3$ case. TPM is easy to implement, but its noise tolerance is restricted, i.e.,\ $\varepsilon \lesssim O(d^{-1})$ with random initialization.  A more recent work \cite{richard2014statistical} proposes using a hybrid of a square unfolding and power iterations to improve the noise tolerance.  This technique is designed for single-spike models (i.e., $r=1$) with Gaussian noise, 
and the sensitivity of the algorithm when some $\lambda_i$ are close to each other has not been explored.
Another decomposition method, orthogonal joint diagonalization (OJD) \cite{kuleshov2015tensor}, seeks all factors simultaneously by joint diagonalization of a set of matrices from random contractions. This approach reaches the full potential in the full rank setting; however, applying OJD to low-rank tensors requires further modifications.

In this paper we propose a new algorithm that reduces the orthogonal decomposition of tensors to that of matrices. Specifically, we unfold the tensor along two modes and demand the left singular vectors of the resulting matrix to be close to a \emph{Kronecker square}~\cite{kofidis2002best}
(i.e., $\matr{x}^{\otimes 2}$ for some vector $\matr{x}$).
By viewing the length-$d^2$ vectors as elements of $\bbR^{d\times d}$, this is equivalent to imposing a ``nearly-rank-1'' constraint in the two-mode singular space.  
We show that the two-mode matrix singular value decomposition (SVD), in conjunction with the nearly-rank-1 constraint, provides an accurate approximation of $\{\matr{u}_i^{\otimes 2}\}$.  We then estimate the underlying factors $\{\matr{u}_i\}$ using the dominant eigenvectors of those nearly rank-1 matrices.  

Our work is inspired by and built on Kruskal's uniqueness theorem \cite{kruskal1977three, sidiropoulos2000uniqueness,bhaskara2014uniqueness}: when $\varepsilon$ vanishes in \eqref{eq:problem}, the set of $\{\matr{u}_i\}$ must be uniquely determined even in the case of degenerate $\lambda_i$s. This is in sharp contrast to the matrix case, where imposing proper constraints on the matrix SVD is needed to avoid the ambiguity caused by degenerate singular values. In fact, the notion of Kronecker-squared singular vector, denoted $\Vec(\matr{u}_i^{\otimes 2})$, is a necessary and sufficient condition for $\matr{u}_i$ to be a desired component in orthogonal tensor decomposition.  Our method ensures the exact recovery in the noiseless case.  Moreover, it achieves very good accuracy at a higher noise level $O(d^{-(k-2)/2})$ than TPM's noise tolerance of $O(d^{-(k-1)/2})$ \cite{richard2014statistical}. 

The remainder of the paper is organized as follows. In Section~\ref{sec:preliminaries}, we introduce some tensor notation and algebra. Section~\ref{sec:SOD} describes the key observation that characterizes the components of orthogonal decomposable tensors.  Our main algorithm and perturbation analysis are summarized in Section~\ref{sec:nearlySOD}.  In Section~\ref{sec:simulation} we demonstrate the accuracy and efficiency of our method through simulations, and in \sref{sec:proofs} we provide
proofs of our theoretical results.

\section{Preliminaries}\label{sec:preliminaries}

We use $\tT =\entry{ \tT_{i_1\ldots i_k}}   \in \bbR^{d_1\times\cdots \times d_k}$ to denote an order-$k$, $(d_1,\ldots, d_k)$-dimensional tensor with entries $\tT_{i_1\ldots i_k}\in\bbR$ where $1 \leq i_n \leq d_n$ for $n = 1,\ldots,k$. A tensor $\tT\in\bbR^{d_1\times \cdots\times d_k}$ is called symmetric if $d_1=\cdots=d_k$ and $\tT_{i_1 \ldots i_k}=\tT_{\pi(i_1)\ldots \pi(i_k)}$ for all permutations $\pi$ of  $\{1,\ldots ,k\}$. We use $[n]$ to denote the $n$-set $\{1,\ldots, n\}$ for $n\in\mathbb{N}_{+}$ and $\mathbf{S}^{d-1}=\{\matr{x}\in\bbR^{d}\colon \vectornorm{\matr{x}}=1\}$ to denote the $(d-1)$-dimensional unit sphere.

A symmetric tensor $\tT$ can be viewed as a multilinear map  \cite{lim2006singular}.  For any  $\matr{x}=(x_1,\ldots,x_d)^T\in\bbR^d$, define 
\[
\tT(\matr{x},\ldots,\matr{x})=\sum_{i_1=1}^d\cdots\sum_{i_k=1}^d \tT_{i_1\ldots i_k}x_{i_1}\cdots x_{i_k}.
\]
The \emph{spectral norm}, or the \emph{$l^2$-norm}, of $\tT$ is defined as 
\[
\norm{\tT}=\sup_{\matr{x}\in\mathbf{S}^{d-1}}|\tT(\matr{x},\ldots,\matr{x})|.
\]
Note that when $k=2$, the tensor $l^2$-norm reduces to the classical matrix $l^2$-norm. 
The Frobenius norm of the tensor $\tT$ is defined as
\[
\Fnorm{\tT} =  \left( \sum_{i_1=1}^d\cdots \sum_{i_k=1}^d \tT^2_{i_1\cdots i_k}\right)^{1/2}.
\]

Given a symmetric order-$k$ tensor $\tT=\entry{\tT_{i_1\ldots i_k} } \in \bbR^{d \times\cdots \times d}$, we can map the indices from $k$-tuples $(i_1,\ldots,i_k)$  to $\ell$-tuples, where $1\leq \ell \leq k$, in various ways to reshape the tensor into lower-order objects. This operation is called \emph{tensor unfolding}. For instance, an order-3 symmetric tensor $\tT$ can be unfolded into a $d$-by-$d^2$ matrix $\tT_{(1)(23)}$ by defining the $(a,b)$-element of the matrix as $\big[\tT_{(1)(23)}\big]_{ab}=\tT_{i_1 i_2 i_3}$, where $a=i_1$ and $b=i_2+(i_3-1)d$. 
The notation $(j_1 \ldots j_m)$ appearing in the subscript of an unfolded tensor denotes that the modes $j_1,\ldots,j_m$ are combined into a single mode.
The following three unfoldings will be of particular interest to us: 
\begin{enumerate}
\item \emph{One-mode unfolding} $\tT_{(1)(2\ldots k)}$, which reshapes $\tT$ into a $d \times d^{k-1}$ matrix.
\item \emph{Two-mode unfolding}  $\tT_{(12)(3\ldots k)}$, which reshapes $\tT$ into a $d^2\times d^{k-2}$ matrix.
\item \emph{Order-3 unfolding}  $\tT_{(1)(2)(3\ldots k)}$, which reshapes $\tT$ into a $d\times d\times d^{k-2}$ cube. 
\end{enumerate}
We now introduce the notion of \emph{two-mode singular space} which is central to our methods. Let $\tT\in\bbR^{d\times\cdots\times d}$ be an order-$k$ symmetric tensor and $\tT_{(12)(3\ldots k)}$ be its two-mode unfolding. We use
\[ 
	\tT_{(12)(3\ldots k)}=\sum_i \mu_i\matr{a}_i\matr{b}^T_i
\]
to denote the {\it two-mode higher-order SVD} (HOSVD); that is, $\mu_i\geq 0$ is the $i$th largest singular value of the matrix $\tT_{(12)(3\ldots k)}$, and $\matr{a}_i\in\bbR^{d^2}$ (respectively, $\matr{b}_i\in\bbR^{d^{k-2}}$) is the $i$th left (respectively, right) singular vector corresponding to $\mu_i$.

\begin{defn} [Two-Mode Singular Space] The $s$-truncated two-mode singular space of $\tT$ is defined by
\[
\ls{LS}^{(s)}=\Spanspace\{\matr{a}_1,\ldots,\matr{a}_s\},
\]
where $\matr{a}_i$ is the $i$th left singular vector of $\tT_{(12)(3\ldots k)}$. 
\end{defn}

\begin{rmk}
Since $\bbR^{d^2}$ is isomorphic to $\bbR^{d\times d}$, we also write $\ls{LS}^{(s)}=\Spanspace\{\text{Mat}(\matr{a}_1), \ldots,\text{Mat}(\matr{a}_s)\}$, where $\text{Mat}(\cdot)$ denotes the matricization operation that unstacks a length-$d^2$ vector into a $d$-by-$d$ matrix. Conversely, we use $\Vec(\cdot)$ to denote the concatenation of matrix columns into a vector. Throughout the paper, we do not distinguish between these two representations of $\ls{LS}^{(s)}$. It should be clear from the context whether we are viewing the elements of the two-mode singular space as length-$d^2$ vectors or $d\times d$ matrices. 
\end{rmk}

\begin{rmk}
If $\mu_s$ is strictly larger than $\mu_{s+1}$, the $s$-truncated space $\ls{LS}^{(s)}$ is uniquely determined. 
\end{rmk}

Tensors arising in applications (such as parameter estimation for latent variable models, independent component analysis, etc) frequently possess special structures. In this paper, we consider a class of tensors that are symmetric and (nearly) orthogonal decomposable. A tensor $\tT$ is called \emph{symmetric and orthogonally decomposable} (SOD,  \cite{mu2015successive}) if $\tT$ can be written as $\tT=\sum_{i=1}^r\lambda_i\matr{u}_i^{\otimes k}$, where $\{ \matr{u}_i\}_{i\in[r]}$ are orthonormal vectors in $\bbR^d$, $\{ \lambda_i\}_{i\in[r]}$ are non-zero scalars in $\bbR$, and $r$ is a positive integer. In this case, each $\matr{u}_i$ is called a \emph{factor} and $r$ the \emph{rank} of the SOD tensor $\tT$. 

Unlike in the matrix case where every symmetric matrix has an eigen-decomposition, symmetric tensors are not necessarily SOD. We are particularly interested in symmetric tensors within a small neighborhood of SOD tensors; i.e.,\ $\ntensor=\sum_{i=1}^r\lambda_i\matr{u}_i^{\otimes k}+\tE$ where $\tE\in\bbR^{d\times \cdots \times d}$ is a symmetric but otherwise arbitrary tensor with $\norm{\tE}\leq \varepsilon$. Here $\tE$ represents the noise tensor arising from, for example, sampling error, model misspecification, finite sample size, etc. We refer to this class of tensors as \emph{nearly-SOD tensors} \cite{mu2015successive}.  For simplicity we will assume $\lambda_i>0$ for all $i\in[r]$ hereinafter. 
The case where $\lambda_i$s have arbitrary signs will be discussed at the end of Section~\ref{sec:nearlySOD}.

\section{Tensor vs. Matrix Decompositions}\label{sec:SOD}

In this section, we restrict our attention to SOD tensors and relate their decompositions to matrix decompositions. The special structure of SOD tensors ($\tT=\sum_i\lambda_i\matr{u}_i^{\otimes k}$, where $\matr{u}_i$ are orthonormal) allows us to illustrate the main idea with a clean and simple proof. We characterize the underlying factors $\{\matr{u}_i\}$ of $\tT$ using rank-1 matrices in the two-mode singular space of $\tT$. As we shall see later, this perspective plays a key role in our algorithm. 
 
\subsection{Characterization of Robust Eigenvectors}
Tensor decompositions possess an interesting uniqueness property not present in matrix decompositions. Consider an SOD tensor, $\tT=\sum_{i=1}^r\lambda_i\matr{u}_i^{\otimes k}$. Kruskal's theorem guarantees that the set of orthonormal vectors $\{\matr{u}_i\}_{i=1}^r$ is unique up to signs even when some $\lambda_i$s are degenerate. A rank $r > 1$ matrix, on the other hand, can be decomposed in multiple manners into a sum of outer-product terms in the case of degenerate $\lambda_i$s. To make the distinction explicit, we refer to these unique components $\{\matr{u}_i\}$ as \emph{robust eigenvectors} \cite{anandkumar2014tensor}.

\begin{defn}[Anandkumar et al \cite{anandkumar2014tensor}] 
A unit vector $\matr{a}\in\bbR^d$ is called a robust eigenvector of $\tT$ if $\matr{a}\in\{\matr{u}_1,\ldots,\matr{u}_r\}$.
\end{defn}

Decomposing an SOD tensor amounts to finding all its robust eigenvectors. Note that the tensor decomposition $\tT=\sum_{i=1}^r\lambda_i\matr{u}_i^{\otimes k}$ implies a series of matrix SVDs, such as $\tT_{(1)(2\ldots k)}=\sum_{i=1}^r\lambda_i\Vec(\matr{u}_i)\Vec(\matr{u}_i^{\otimes k-1})^T$, $\tT_{(12)(3\ldots k)}=\sum_{i=1}^r\lambda_i\Vec(\matr{u}_i^{\otimes 2})\Vec(\matr{u}_i ^{\otimes k-2})^T, \ldots,$ and so on. 
This suggests a way to use matrix SVDs to recover $\{\matr{u}_i\}$. Despite the seeming simplicity, however, matrix SVDs may lead to spurious solutions because they are not guaranteed to be unique. For example, in the case of $\lambda_1=\lambda_2$, $\matr{u}_1+\matr{u}_2$ is a (unnormalized) left singular vector of $\tT_{(1)(2\ldots k)}$. Clearly, $\matr{u}_1+\matr{u}_2$ is non-parallel to any robust eigenvector of $\tT$. Similarly, $\Vec(\matr{u}_1^{\otimes 2}+\matr{u}_2^{\otimes 2})$ is a (unnormalized) left singular vector of $\tT_{(12)(3\ldots k)}$ without being parallel to any $\Vec(\matr{u}_i^{\otimes 2})$.

Fortunately, such spurious solutions can be ruled out by enforcing a certain Kronecker-product constraint on matrix singular vectors. Specifically, we have the following characterization of robust eigenvectors:

\begin{thm}\label{thm:characterization}
A unit vector $\matr{a}\in\bbR^d$ is a robust eigenvector of $\tT$ if and only if $\Vec(\matr{a}^{\otimes 2})$ is a left singular vector of $\tT_{(12)(3\dots k)}$ corresponding to a non-zero singular value. 
\end{thm}

A proof of Theorem~\ref{thm:characterization} is provided in \sref{appendix:proofcharacterization}. Theorem~\ref{thm:characterization} is closely related to the uniqueness property and it provides a criterion for a unit vector $\matr{a}$ to be a robust eigenvector of $\tT$. 
An earlier work~\cite{anandkumar2014tensor} has shown that robust eigenvectors can be characterized using local maximizers of the objective function $\matr{a}\mapsto  \tT(\matr{a},\ldots,\matr{a})/\vectornorm{\matr{a}}^k$. Our result provides a new perspective that does not require checking the gradient and/or the Hessian of the objective function.

\subsection{Exact Recovery for SOD Tensors}
For an SOD tensor $\tT=\sum_{i=1}^r\lambda_i\matr{u}_i^{\otimes k}$, we define an $r$-dimensional linear space $\ls{LS}_0\eqdef\Spanspace\{\matr{u}^{\otimes 2}_i: i\in[r]\}$.  The elements of $\ls{LS}_0$ can be viewed as either length-$d^2$ vectors or $d$-by-$d$ matrices.
The space $\ls{LS}_0$ is exactly recovered by $\ls{LS}^{(r)}$, the $r$-truncated two-mode left singular space of $\tT$, where $r$ equals the rank of $\tT_{(12)(3\ldots k)}$. Similar to Theorem~\ref{thm:characterization}, imposing rank-1 constraints ensures the desired solutions in $\ls{LS}_0$:

\begin{prop} \label{prop:searchspace}
Every rank-1 matrix in $\ls{LS}_0$ is (up to a scalar) the Kronecker square of some robust eigenvector of $\tT$.
\end{prop}

Proposition \ref{prop:searchspace} implies that rank-1 matrices in $\ls{LS}_0$ are sufficient to find $\{\matr{u}_i\}$.  Note that  a matrix $\matr{M}\in\bbR^{d\times d}$ being rank-1 is equivalent to $\norm{\matr{M}}/\Fnorm{\matr{M}}=1$, since
$\norm{\matr{M}} \leq \Fnorm{\matr{M}} \leq \sqrt{\text{rank $\matr{M}$}} \norm{\matr{M}}$.  Hence, Proposition~\ref{prop:searchspace} immediately suggests the following algorithm:
\begin{equation}\label{eq:noiseless}
\begin{aligned}
&\maximize_{\matr{M} \in\bbR^{d\times d}} \norm{\matr{M}},\\
&\text{subject to } \matr{M}\in \ls{LS}_0 \text{ and }\Fnorm{\matr{M}}=1.
\end{aligned}
\end{equation}

\begin{thm}[Exact Recovery in the Noiseless Case]\label{thm:noiseless}
The optimization problem \eqref{eq:noiseless} has exactly $r$ pairs of local maximizers $\{\pm \matr{M}^*_i\colon i\in[r]\}$. Furthermore, they satisfy the following three properties:
\renewcommand{\leftmargini}{9mm}
\begin{enumerate}[label=(A\arabic*)]
\item $\norm{\matr{M}^*_i}=1$ for all $i\in[r]$. \label{A1}
\item $\left|\langle \Vec(\matr{M}^*_i),\ \Vec(\matr{M}^*_j) \rangle\right|=\delta_{ij}$ for all $i,j\in[r]$, where
$\langle \cdot,\cdot\rangle$ denotes the inner product.
\item There exists a permutation $\pi$ on $[r]$ such that $\matr{M}^*_i=\pm \matr{u}_{\pi(i)}^{\otimes 2}$ for all $i\in[r]$. \label{A3}
\end{enumerate}
\end{thm}
A proof is provided in \sref{appendix:proofnoiseless}.  Since $\norm{\matr{M}} \leq \Fnorm{\matr{M}}$ in general and $\matr{M}$ is constrained to satisfy
$\Fnorm{\matr{M}}=1$, property~\ref{A1} implies that every local maximizer $\matr{M}^*_i$ is a global maximizer. Therefore, any algorithm that ensures local optimality is able to  recover exactly the set of matrices $\{\matr{u}^{\otimes 2}_i\}$. As a by-product, the number $r$ of factors is recovered by the number of linearly independent solutions $\{\matr{M}^*_{i}\}$. In addition, property~\ref{A3} indicates $\{\matr{u}_i\}$ can be extracted from the dominant eigenvectors of the optimal solutions $\{\matr{M}^*_i\}$. 

\section{Two-Mode HOSVD via Nearly Rank-1 Matrix Pursuit}\label{sec:nearlySOD}
In most statistical and machine learning applications, the observed SOD tensors are perturbed by noise.  In this section we extend the results in Section~\ref{sec:SOD} to nearly-SOD tensors, which take the form 
\begin{equation}\label{eq:model}
\ntensor=\sum_{i=1}^r\lambda_i\matr{u}_i^{\otimes k}+\tE,
\end{equation}
where the first part on the right hand side is SOD and $\tE\in\bbR^{d\times \cdots \times d}$ is a symmetric but otherwise arbitrary noise tensor satisfying $\norm{\tE}\leq \varepsilon$. Our goal is to estimate the underlying pairs $\{(\matr{u}_i,\lambda_i)\in\bbR^d\times \bbR\}_{i\in[r]}$ from $\tensor{\widetilde T}$.  We note that the number of factors $r \in \mathbb{N}_{+}$ is typically unknown and has to be determined empirically. 
As in most previous works on tensor decomposition, we assume that $r$ is known in Sections~\ref{section:algorithm} and \ref{section:perturbation}. In Section~\ref{section:number}, we describe a rule of thumb for choosing $r$.

\subsection{Algorithm}\label{section:algorithm}
An orthogonal decomposition of a noisy tensor $\ntensor$ does not necessarily exist, and therefore finding rank-1 matrices in the two-mode singular space of $\ntensor$ may not be possible.  Nevertheless, the formulation shown in \eqref{eq:noiseless} suggests a practical way to approximate $\{\matr{u}_i\}$. Specifically, we seek to find a ``nearly'' rank-1 matrix $\matr{\widehat{M}}$ via the following optimization:
\begin{equation}\label{eq:noisy}
\begin{aligned}
&\maximize_{\matr{M} \in\bbR^{d\times d}} \norm{\matr{M}},\\
&\text{subject to } \matr{M}\in \ls{LS}^{(r)} \text{ and }\Fnorm{\matr{M}}=1,
\end{aligned}
\end{equation}
and take the dominant eigenvector $\matr{\widehat u}=\argmax_{\matr{x}\in\mathbf{S}^{d-1}} |\matr{x}^T\matr{\widehat M}\matr{x}|$ as an estimator of $\matr{u}_i$. Recall that $\ls{LS}^{(r)}=\Spanspace\{\matr{a}_1,\ldots, \matr{a}_r\}$, where $\matr{a}_i$ is the $i$th left singular vector of the two-mode unfolding $\ntensor_{(12)(3\ldots k)}$. In principle, the $r$-truncated two-mode singular space $\ls{LS}^{(r)}$ might not be uniquely determined; for example, this occurs when the $r$th and $(r+1)$th singular values of $\ntensor_{(12)(3\ldots k)}$ are equal. We shall make the following assumption on the model \eqref{eq:model} to ensure the uniqueness of $\ls{LS}^{(r)}$:

 \begin{assume}[Signal-to-Noise Ratio]\label{assumption}
In the notation of model~\eqref{eq:model}, assume $\varepsilon\leq  \lambda_{\min}/\left[c_0d^{(k-2)/2}\right]$, where $\lambda_{\min}=\min_{i\in[r]}\lambda_i$ and $c_0>2$ is a constant that does not depend on $d$. 
\end{assume}

We refer to $\lambda_{\min}/\varepsilon$ as the signal-to-noise ratio (SNR) of $\ntensor$, because $\lambda_{\min}$ represents the minimum eigenvalue of the signal tensor $\sum_{i=1}^r\lambda_i\matr{u}_i^{\otimes k}$, whereas $\varepsilon$ represents the maximum eigenvalue of the noise tensor $\tE$. 
In \sref{appendix:proofuniqueness}, we prove the following result:
\begin{prop}\label{prop:uniqueness}
Under Assumption~\ref{assumption}, the two-mode singular space $\ls{LS}^{(r)}$ is uniquely determined. 
\end{prop}
Assumption \ref{assumption} implies that the threshold SNR scales as $O(d^{(k-2)/2})$. Although it may appear stringent, this assumption is prevailingly made by most computationally tractable algorithms \cite{richard2014statistical}. In fact, for order-3 tensors, our scaling $O(\sqrt{d})$ is less stringent than the $O(d)$ required by the power iteration approach \cite{anandkumar2014tensor}.

\begin{algorithm}[t]
\hspace*{\SpaceReservedForComments}{}%
\begin{minipage}{\dimexpr\linewidth-\SpaceReservedForComments\relax}
  \caption{Two-mode HOSVD}\label{alg:main}
  \begin{algorithmic}[1]
  \INPUT Noisy tensor $\ntensor$ where $\ntensor=\sum_{i=1}^r \lambda_i\matr{u}^{\otimes k}_i+\tE$, number of factors $r$.
  \OUTPUT $r$ pairs of estimators $(\matr{\widehat u}_i,\widehat \lambda_i)$.
  \vspace{.1cm}
  \State Reshape the tensor $\ntensor$ into a $d^2$-by-$d^{k-2}$ matrix $\ntensor_{(12)(3\ldots k)}$;
  \State Find the top $r$ left singular vectors of $\ntensor_{(12)(3\ldots k)}$, denoted $\{\matr{a}_1,\ldots,\matr{a}_r\}$;
  \State {\bf Initialize} $\ls{LS}^{(r)}=\Spanspace\{\matr{a}_i\colon i\in[r]\}$;
  \For {$i$=1 to r}
  \State Solve $\matr{\widehat M}_i=\argmax\limits_{\matr{M}\in\ls{LS}^{(r)},\Fnorm{\matr{M}}=1}\norm{\matr{M}}$ and $\matr{\widehat u}_i=\argmax\limits_{\matr{u}\in\mathbf{S}^{d-1}}|\matr{u}^T\matr{\widehat M}_i\matr{u}|$;
  \State Update $\matr{\widehat M}_i\leftarrow \ntensor_{(1)(2)(3\ldots k)}(\matr{I},\matr{I},\Vec(\matr{\widehat u}_i^{\otimes (k-2)}))$ and $\matr{\widehat u}_i\leftarrow \argmax\limits_{\matr{u}\in\mathbf{S}^{d-1}}|\matr{u}^T\matr{\widehat M}_i\matr{u}|$;
  \State Return $(\matr{\widehat u}_i,\widehat \lambda_i)\leftarrow (\matr{\widehat u}_i,\ntensor(\matr{\widehat u}_i,\ldots,\matr{\widehat u}_i))$;
  \State Set $\ls{LS}^{(r)}\leftarrow \ls{LS}^{(r)}\cap \left[\Vec(\matr{\widehat u}_i^{\otimes 2})\right]^\perp$;
  \EndFor
  \vspace{-.4cm}
  \end{algorithmic}
\AddNote[black]{1}{2}{\small Two-Mode HOSVD}
\AddNote[black]{5}{5}{\small Nearly Rank-1 Matrix}
\AddNote[black]{6}{7}{\small Post-Processing}
\AddNote[black]{8}{8}{\small Deflation}
\end{minipage}%
\end{algorithm}

Algorithm~\ref{alg:main} outlines our method, which we divide into four parts for ease of reference. The algorithm consists of $r$ successive iterations. At each iteration, we search for a nearly rank-1 matrix $\matr{\widehat M}$ in the space $\ls{LS}^{(r)}$ and compute the top eigenvector of $\matr{\widehat M}$ (or some refined version of $\matr{\widehat M}$). We then deflate the space $\ls{LS}^{(r)}$ and repeat the procedure until a full decomposition is obtained.

In the noiseless case, Algorithm~\ref{alg:main} guarantees the exact recovery of $\{\matr{u}_i\}$. Further study demonstrates that this approach also accurately approximates $\{\matr{u}_i\}$ in the presence of noise satisfying Assumption~\ref{assumption}. Here we discuss a few critical algorithmic aspects 
of the four steps in Algorithm~\ref{alg:main} and defer theoretical analyses to the next section. 
\medskip

{\bf Two-Mode HOSVD}. We perform the $r$-truncated SVD on the two-mode unfolded matrix $\ntensor_{(12)(3\ldots k)}\in\mathbb{R}^{d^2\times d^{k-2}}$. This step typically requires $O(d^kr)$ floating-point operations (flops). When $\ntensor_{(12)(3\ldots k)}$ is a ``fat'' matrix, the computational cost can be reduced by finding the $r$-truncated eigen decomposition of the $d^2$-by-$d^2$ matrix $[\ntensor_{(12)(3\ldots k)}][\ntensor_{(12)(3\ldots k)}]^T$. The latter approach involves $O(d^k)$ flops for matrix multiplication and $O(d^4r)$ flops for eigen decomposition. Hence, the total cost of this component is $O(d^kr)$ for $k=3$, $4$ and $O(d^k)$ for $k\geq 5$.

\medskip
{\bf Nearly Rank-1 Matrix}. This step of the algorithm requires optimization, and the computational cost depends on the choice of the optimization subroutine. For $r=1$ (i.e.,\ single-spike model), no actual optimization is needed since ${\matr{\widehat M}}_1\eqdef\arg\max_{\matr{M}\in\ls{LS}^{(1)},\Fnorm{M}=1}\norm{\matr{M}}$ is simply $\text{Mat}(\matr{a}_1)$. For $r\geq 2$, there exist various subroutines for optimizing $\norm{\matr{M}}$. Here we choose to use coordinate ascent with initialization $\matr{\widehat M}_i^{(0)}=\text{Mat}(\matr{a}_i)$ at the $i$th iteration. Briefly, we introduce two decision variables $\matr{x}\in\mathbf{S}^{d-1}$ and $\matr{\alpha}\in\mathbf{S}^{r-1}$, and rewrite \eqref{eq:noisy} as $\max_{\matr{x}\in\mathbf{S}^{d-1},\matr{\alpha}\in\mathbf{S}^{r-1}}H(\matr{x},\matr{\alpha})$, where $H(\matr{x},\matr{\alpha})\eqdef \matr{x}^T[\alpha_1\text{Mat}(\matr{a}_i)+\cdots+\alpha_r\text{Mat}(\matr{a}_r)]\matr{x}$. Since both $\max_{\matr{x}\in\mathbf{S}^{d-1}}H(\matr{\matr{x}},\cdot)$ and $\max_{\matr{\alpha}\in\mathbf{S}^{r-1}}H(\cdot,\matr{\alpha})$ have closed-form solutions, we update $\matr{x}$ and $\matr{\alpha}$ alternatively until convergence. The computational cost is $O(d^2)+O(d^2r)$ for each update. 
Upon finding $\matr{\widehat M}_i$, we set $\matr{\widehat u}_i=\argmax_{\matr{u}\in\mathbf{S}^{d-1}}|\matr{u}^T\matr{\widehat M}_i\matr{u}|$.

Although the above procedure finds only a local maximum, in practice we found that the converged points generally have nearly optimal objective values. In fact, local optimality itself is not a severe issue in our context, because we seek a total of top $r$ maximizers in $\ls{LS}^{(r)}$ but the order in which they are found is unimportant. When the noise is small enough, at some iterations there could be multiple close-to-optimal choices of $\matr{\widehat M}_i$, with negligible difference between their objective values. In that case,  any of these choices performs equally well in estimating $\matr{u}_i$ and it is thus of little interest to identify a true global optimum.

\medskip
{\bf Post-Processing}. In this step we update $ \matr{\widehat M}_i\leftarrow \ntensor_{(1)(2)(3\ldots k)}(\matr{I},\matr{I},\Vec(\matr{\widehat u}_i^{\otimes (k-2)}))$ and take its top eigenvector as an estimator of $\matr{u}_i$. Here, $\matr{I}$ denotes the $d$-by-$d$ identity matrix. This step requires $O(d^k)$ flops. The updated matrix $\matr{\widehat M}_i$ is arguably still close to rank-1 and provides a better estimation of $\matr{u}_i$. This post-processing is intended to make the algorithm more robust to higher-order noise. 
Intuitively, merging modes $(3,\ldots,k)$ in the earlier stage of the algorithm amplifies the noise (measured by spectral norm) because of the corresponding loss in multilinear structure \cite{wang2016operator}. The updated matrix $\matr{\widehat M}_i=\ntensor(\matr{I},\matr{I},\matr{\widehat u}_i,\ldots,\matr{\widehat u}_i)$ helps to alleviate such an effect because it essentially decouples the combined modes. The theoretical analysis in the next section confirms this intuition. 

\medskip
{\bf Deflation}. Our deflation strategy is to update the search space using the orthogonal complement of $\Vec(\matr{\widehat u}_i^{\otimes 2})$ in $\ls{LS}^{(r)}$, where $\matr{\widehat u}_i$ is the estimator in the current iteration. This can be done by setting the basis vector $\matr{a}_j \leftarrow \matr{a}_j -\langle \Vec(\matr{\widehat u}_i^{\otimes 2}), \matr{a}_j \rangle \Vec(\matr{\widehat u}_i^{\otimes 2})$ for all $j\in [r]$, followed by normalizing $\{\matr{a}_j\}_{j\in[r]}$. The complexity involved is $O(d^2 r)$. An alternative deflation strategy is to update the tensor $\ntensor\leftarrow \ntensor-\widehat \lambda_i\matr{\widehat u}_i^{\otimes k}$ and then repeat lines 1--7 of Algorithm~\ref{alg:main}. A careful analysis reveals that both approaches effectively control the accumulated error. We adopt the former strategy because it allows us to avoid recalculating of the two-mode HOSVD. 

\medskip
The total cost of our algorithm is $O(d^k)$ per iteration. This is comparable to or even lower than that of competing methods. For order $k=3$, TPM has complexity $O(d^3M)$ per iteration (where $M$ is the number of restarts), and OJD has complexity $O(d^3L)$ per iteration (where $L$ is the number of matrices). Furthermore, we have provided theoretical results for an arbitrary order $k$, while neither the TPM nor the OJD method provided a theoretical analysis for higher-order cases.

\subsection{Theoretical Analysis}\label{section:perturbation}
In this section we theoretically analyze the accuracy of the estimators $\{(\matr{\widehat u}_i,\widehat \lambda_i)\in\bbR^d\times \bbR\}_{i\in[r]}$ produced by our algorithm. We first consider extracting one pair $(\matr{\widehat u}_1,\widehat \lambda_1)$ by Algorithm~\ref{alg:main}  and provide the corresponding error bounds. We then show that the accumulated error caused by deflation has only a higher-order effect and is thus negligible. The stability for the full decomposition is rendered through a combination of these two results. 

We develop a series of algebraic techniques based on multilinear functional analysis, Weyl's and Wedin's perturbation theorems~\cite{weyl1949inequalities, wedin1972perturbation}, some of which may be of independent interest. Here we only state the key theorems and lemmas. All proofs are deferred to \sref{sec:proofs}.

Following \cite{richard2014statistical}, we use a loss function that allows the sign-flip error. 

\begin{defn} \label{def:loss}
Let $\matr{a}$, $\matr{b}\in \bbR^d$ be two unit vectors, and define
\[
\Loss(\matr{a},\matr{b})=\min\left(\vectornorm{\matr{a}-\matr{b}}, \vectornorm{\matr{a}+\matr{b}}\right).
\]
If $a$, $b$ are two scalars in $\bbR$, we define $\Loss(a,b)=\min\left(|a-b|, |a+b|\right).$
\end{defn}

The following lemma describes the deviation of the perturbed space $\ls{LS}^{(r)}$ from the true space $\ls{LS}_0$ under small noise. The result can be viewed as an analogue of Wedin's perturbation in the context of two-mode HOSVD. 
\begin{thm}[Perturbation of $\ls{LS}_0$]\label{thm:LS}
Suppose $c_0\geq 10$ in Assumption~\ref{assumption}. Then, 
\begin{equation}\label{eq:pertLS}
\max_{
\substack{
\matr{M}\in\ls{LS}^{(r)},\\
\Fnorm{\matr{M}}=1
}
}\; 
\min_{\matr{M^*}\in\ls{LS}_0}\norm{\matr{M}-\matr{M^*}} \leq {d^{(k-3)\over 2}\varepsilon\over \lambda_{\min}}+o(\varepsilon).
\end{equation}
\end{thm}

We choose to use the matrix spectral norm to measure the closeness of $\ls{LS}^{(r)}$ and $\ls{LS}_0$. Although the Frobenius norm may seem an easier choice, it poorly reflects the multilinear nature of the tensor and leads to suboptimal results. Note that the dimension factor vanishes when $k=3$ (i.e.,\ order-3 tensors).

The following lemma guarantees the existence of a nearly rank-1 matrix in the perturbed space $\ls{LS}^{(r)}$.

\begin{lem}[Nearly Rank-1 Matrix] \label{lem:max}
Under Assumption \ref{assumption}, we have
\begin{equation}\label{eq:maxrank}
\max_{
\substack{
\matr{M}\in\ls{LS}^{(r)},\\
\Fnorm{\matr{M}}=1
}
} \norm{\matr{M}}\geq 1-{d^{(k-2)\over 2}\over \lambda_{\min}}\varepsilon+o(\varepsilon).
\end{equation}
\end{lem}
In fact, we can show that there exist at least $r$ linearly independent matrices (here we view matrices as elements of $\bbR^{d^2}$) that satisfy the above inequality. 
More generally, at $i$th iteration of Algorithm~\ref{alg:main}, there are at least $r-i+1$ linearly independent nearly rank-1 matrices in the search space.

The following lemma provides an error bound for the estimator $\matr{\widehat u}_i$ extracted from a nearly rank-1 matrix in $\ls{LS}^{(r)}$. 
\begin{lem}[Error Bound for the First Component]\label{lem:bound}
Let $\matr{\widehat M}_1$ and $\matr{\widehat u}_1$ respectively be the matrix and the vector defined in line 5 of Algorithm~\ref{alg:main}. Suppose $c_0\geq 10$ in Assumption~\ref{assumption}. Then, there exists $i^*\in[r]$ such that
\begin{equation}\label{eq:boundmaintext}
\Loss(\matr{\widehat u}_1,\matr{u}_{i^*}) \leq  {d^{(k-3)\over 2}\varepsilon\over \lambda_{\min}}+o(\varepsilon).
\end{equation}
\end{lem}

We note that in Lemma~\ref{lem:bound}, the perturbation bound for the estimator does not depend on the eigen-gaps (i.e., $|\lambda_i-\lambda_j|$).
This feature fundamentally distinguishes tensor decompositions from matrix decompositions. Also, note that an amplification of the error (by a polynomial factor of $d$) is observed in the bound. The following lemma suggests a simple post-processing step which improves the error bound.

\begin{lem}[Post-Processing] \label{lem:post-processing}
Let $\matr{\widehat M}_1$ and $\matr{\widehat u}_1$ respectively be the matrix and the vector defined in line 6 of Algorithm~\ref{assumption}.  Further, define $\widehat \lambda_1=\ntensor(\matr{\widehat u}_1,\ldots, \matr{\widehat u}_1)$ and let $i^*\in[r]$ be the index that appears in Lemma \ref{lem:bound}. 
Suppose $c_0\geq \max\big\{10,{3(k-2)\over 2}+{6\lambda_{\max}\over \lambda_{\min}}\big\}$ in Assumption~\ref{assumption}. Then, 
\begin{equation}\label{eq:firstbound}
\Loss(\matr{\widehat u}_1,\matr{u}_{i^*}) \leq {\varepsilon \over \lambda_{i^*}}+o(\varepsilon), \quad
\Loss(\widehat\lambda_1,\lambda_{i^*}) \leq \varepsilon+o(\varepsilon).
\end{equation}
\end{lem}

Compared to Lemma~\ref{lem:bound}, the leading terms in the above error bounds no longer scale with $d$. 
Finally, a careful analysis shows that the estimation error does not amplify through deflation:
\begin{lem} [Deflation of the Singular Space]\label{lem:deflation}
Suppose $c_0\geq \max\{10, {3(k-2)\over 2}+{6\lambda_{\max}\over \lambda_{\min}}\}$ in Assumption~\ref{assumption}, and 
for a fixed subset $X\subset[r]$, suppose there exists a set $\{\matr{\widehat u}_i\}_{i\in X}$ of unit vectors satisfying
\[
\Loss(\matr{\widehat u}_i,\matr{u}_{\pi(i)})\leq {2\varepsilon\over \lambda_{\pi(i)}}+o(\varepsilon),\quad \text{for all } i\in X,
\]
where $\pi$ is a permutation on $[r]$. 
Then,
\[
\max_{\substack{\matr{M}\in \ls{LS}^{(r)}(X),\\ \Fnorm{\matr{M}}=1}} \; \min_{\matr{M}^*\in\ls{LS}_0(X)} \norm{\matr{M}-\matr{M}^*}\leq {2d^{(k-3)\over 2}\varepsilon\over \lambda_{\min}}+o(\varepsilon),
\]
where $\ls{LS}^{(r)}(X)$ and $\ls{LS}_0(X)$ are residual spaces defined as $\ls{LS}^{(r)}(X)\eqdef\ls{LS}^{(r)}\cap\Spanspace\{ \matr{\widehat u}_i^{\otimes 2}: i\in X\}^{\perp}$ and $\ls{LS}_0(X)\eqdef\Spanspace\{\matr{u}_{\pi(i)}^{\otimes 2}\colon i\in [r]\backslash X\}$. 
\end{lem}
The above result demonstrates that the error due to deflation is bounded by a factor that does not depend on the iteration number. This implies that our proposed deflation strategy is numerically stable when extracting the subsequent components. Similarly, if we define the residual tensor, $\ntensor(X)\eqdef\tensor{\widetilde T}-\sum_{i\in X}\widehat \lambda_i \matr{\widehat u}_i^{\otimes k}$, one can show \cite{mu2015successive} that $\norm{\ntensor(X)-\sum_{i\in [r]\backslash X}\lambda_i\matr{u}_i^{\otimes k}}\leq c\varepsilon+o(\varepsilon)$, where $c>0$ is a constant that dose not depend on the iteration number. The above observations immediately suggest two possible ways to find the subsequent factors, by deflating either the singular space or the original tensor. Note that the deflation results for tensors are non-trivial since the analogous statement for matrices is not necessarily true \cite{anandkumar2014tensor, mu2015successive}.

Applying Lemmas~\ref{lem:max}--\ref{lem:deflation} successively, we obtain the main result of this section as follows. 
\begin{thm} \label{thm:main}
Let $\ntensor=\sum_{i=1}^r\lambda_i \matr{u}_i^{\otimes k}+\tE\in\bbR^{d\times\cdots \times d}$, where $\{\matr{u}_i\}_{i\in[r]}$ is a set of
orthonormal vectors in $\bbR^d$, $ \lambda_i>0$ for all $i\in[r]$, and $\tE$ is a symmetric tensor satisfying $\norm{\tE}\leq \varepsilon$. Suppose $\varepsilon\leq \lambda_{\min}/\left[c_0d^{(k-2)/2}\right]$, where $c_0>0$ is a sufficiently large constant that does not depend on $d$. Let $\{(\matr{\widehat u}_i,\widehat \lambda_i)\in\bbR^d\times \bbR\}_{i\in[r]}$ be the output of Algorithm~\ref{alg:main} for inputs $\ntensor$ and $r$. Then, there exists a permutation $\pi$ on $[r]$ such that for all $i\in[r]$,
\begin{equation}\label{eq:finalconclusion}
\Loss(\matr{\widehat u}_i,\matr{u}_{\pi(i)}) \leq {2\varepsilon\over \lambda_{\pi(i)}}+o(\varepsilon),\quad \Loss(\widehat \lambda_i, \lambda_{\pi(i)}) \leq 2\varepsilon+o(\varepsilon),
\end{equation}
and
\begin{equation}\label{eq:totalbound}
\normSize{\bigg}{\ntensor-\sum_{i=1}^r\widehat\lambda_i\matr{\widehat u}_i^{\otimes k}}\leq C\varepsilon+o(\varepsilon),
\end{equation}
where $C=C(k)>0$ is a constant that only depends on $k$.  
\end{thm}
\begin{rmk}
With little modification, our results also apply to the case when not all $\lambda_i$ are positive. In that case, Assumption~\ref{assumption} needs to be modified to $\varepsilon\leq |\lambda|_\text{\rm min}/\left[ c_0d^{(k-2)/2}\right]$ where $|\lambda|_\text{\rm min}=\min_{i\in[r]}{|\lambda_i|}$.
\end{rmk}

\subsection{Determining the Number of Factors}\label{section:number}
Determining the rank of a general tensor is intrinsically difficult \cite{hillar2013most}. For nearly-SOD tensors, however, $r$ can be approximated in various ways. 
When the number of factors $r$ is unknown, we relax the search space $\ls{LS}^{(r)}$ to $\ls{LS}^{(n)}$ where $n\eqdef \min\{\text{Rank}(\ntensor_{(12)(3\ldots k)}), d\}$. Note that we always have $\ls{LS}^{(n)}\supset \ls{LS}^{(r)}$ under Assumption~\ref{assumption}. Algorithm~\ref{alg:number} describes a rule of thumb for choosing $r$.

\begin{algorithm}[t]
  \caption{Determining the Number of Factors}\label{alg:number}
  \begin{algorithmic}[1]
\INPUT Noisy tensor $\ntensor$ where $\ntensor=\sum_{i=1}^r \lambda_i \matr{u}_i^{\otimes k}+\tE$;
\OUTPUT The number of factors $\widehat r$.

  \State Run Algorithm~\ref{alg:main} with $n$ iterations. Let $(\matr{\widehat M}_i,\matr{\widehat u}_i,\widehat \lambda_i)\in\bbR^{d\times d}\times \bbR^d\times \bbR$ denote the output from the $i$th iteration where $i\in[n]$;
    
 \State  Choose the subset $\tensor{S}\subset [n]$ for which $\normSize{\big}{\matr{\widehat M}_i}$ (i.e.,\ the objective value) is close to 1;

\State Sort $\{{\widehat \lambda_i}^2\colon i\in\tensor{S}\}$ in decreasing order. 

Pick $r$ by the ``elbow'' method of the scree plot (similarly as in matrix PCA).

\end{algorithmic}
\end{algorithm}

Although Algorithm~\ref{alg:number} is heuristic, our simulation results suggest that it provides a good approximation in most cases encountered. The use of $\widehat \lambda_i^2$ as a guiding criterion is justified by the following observation: $\FnormSize{\big}{\ntensor-\sum_i\widehat \lambda_i\matr{\widehat u}_i^{\otimes k}}^2\approx \FnormSize{\big}{\ntensor}^2-\sum_i\widehat \lambda^2_i$, provided that $\{\matr{\widehat u}_i\}$ are approximately orthogonal to each other.

\section{Numerical Experiments}\label{sec:simulation}

We assessed the performance of our algorithm by simulating tensors of order $k=3$, $4$, and $5$.  We generated nearly-SOD tensors $\ntensor=\sum_{i=1}^r\lambda_i\matr{e}_i^{\otimes k}+ \tE \in \bbR^{d\times\cdots\times d}$, where $\matr{e}_i$ is the $i$th canonical basis vector and $\lambda_i$s are i.i.d.\ draws from $\text{Unif}[0.8,1.2]$. The noise tensor $\tE$ was generated with noise level $\sigma$ under one of the following three random models:

\begin{enumerate}
\item (Gaussian) For $i_1\leq \cdots \leq i_k,$ draw independent entries $\tE_{i_1 \ldots i_k}\sim \tensor{N}(0,\sigma^2)$ uniformly at random. 
\item (Bernoulli) For $i_1\leq \cdots \leq i_k,$ draw independent entries $\tE_{i_1\ldots i_k}=\pm \sigma$ with probability $1/2$ for each. 
\item  (Student's $t$-distribution) For $i_1\leq \cdots \leq i_k,$ draw independent entries $\tE_{i_1\cdots i_k} \sim \sigma t(5)$. We use $t$-distribution with 5 degrees of freedom to mimic heavy-tailed data in real-world applications. 
\end{enumerate}

Condition on the i.i.d.\ entries $\{ \tE_{i_1\ldots i_k}\}_{i_1 \leq \cdots \leq i_k}$, we generated the remaining entries by imposing the symmetry condition $\tE_{i_1\ldots i_k}=\tE_{\pi(i_1)\ldots \pi(i_k)}$, where $(i_1,\ldots, i_k)\in[d]^k$ and $\pi$ is a permutation of $[k]$.  We set the dimension $d\in\{25,50\}$ and the number of factors $r\in\{2, 10, 25\}$, and varied $\sigma$. For each scenario, we generated $50$ trials and applied Algorithm~\ref{alg:main} to obtain the estimators $\{\matr{\widehat u}_i,\widehat \lambda_i\}_{i\in[r]}$. The estimation accuracy was assessed by the average $l^2$ loss ${1\over r}\sum_{i=1}^r\Loss(\matr{\widehat u}_i, \matr{u}_{\pi(i)})$ across the 50 trials.  We compared the two-mode HOSVD (TM-HOSVD) with the orthogonal joint diagonalization (OJD) method \cite{kuleshov2015tensor} and the tensor power method (TPM)~\cite{anandkumar2014tensor} with random initialization. The software packages for TPM and OJD were originally designated for $k=3$, so we extended them to $k\geq 4$. Since both TPM and OJD require users to specify the number of factors, we provided them with the ground-truth $r$.

\begin{figure}[t]
\centerline{\includegraphics[width=14cm]{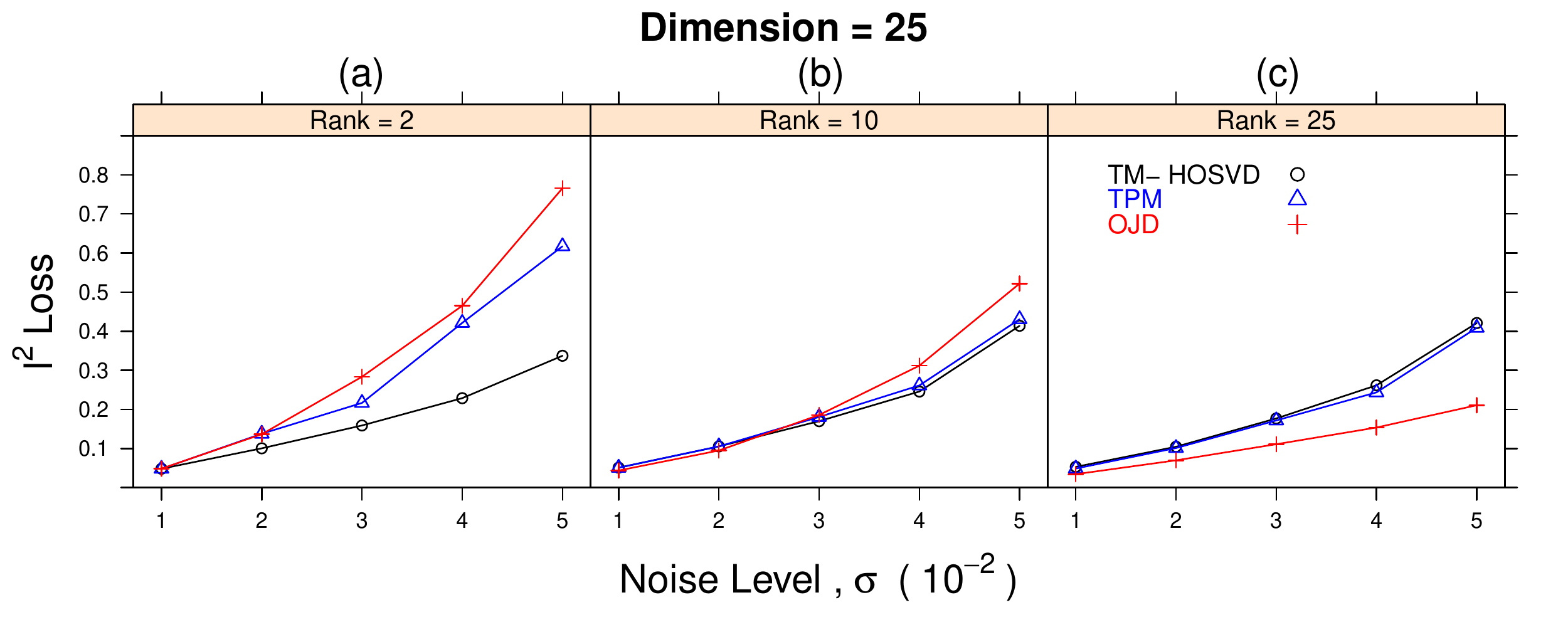}}
\centerline{\includegraphics[width=14cm]{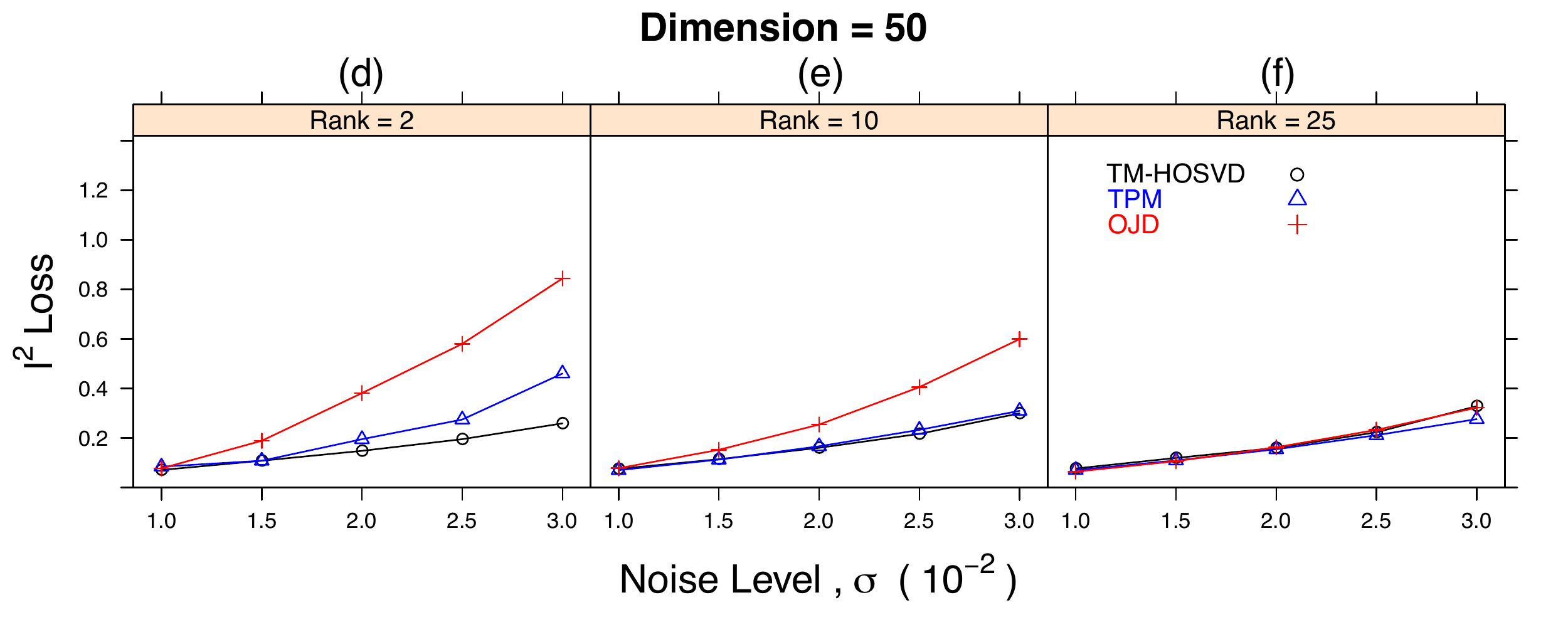}}
 \caption{Average $l^2$ loss for decomposing order-3 nearly SOD tensors with Gaussian noise.}
 \label{fig:gaussian}
\end{figure}

\begin{figure}[t]
 \begin{center}
  \includegraphics[width=5.6cm]{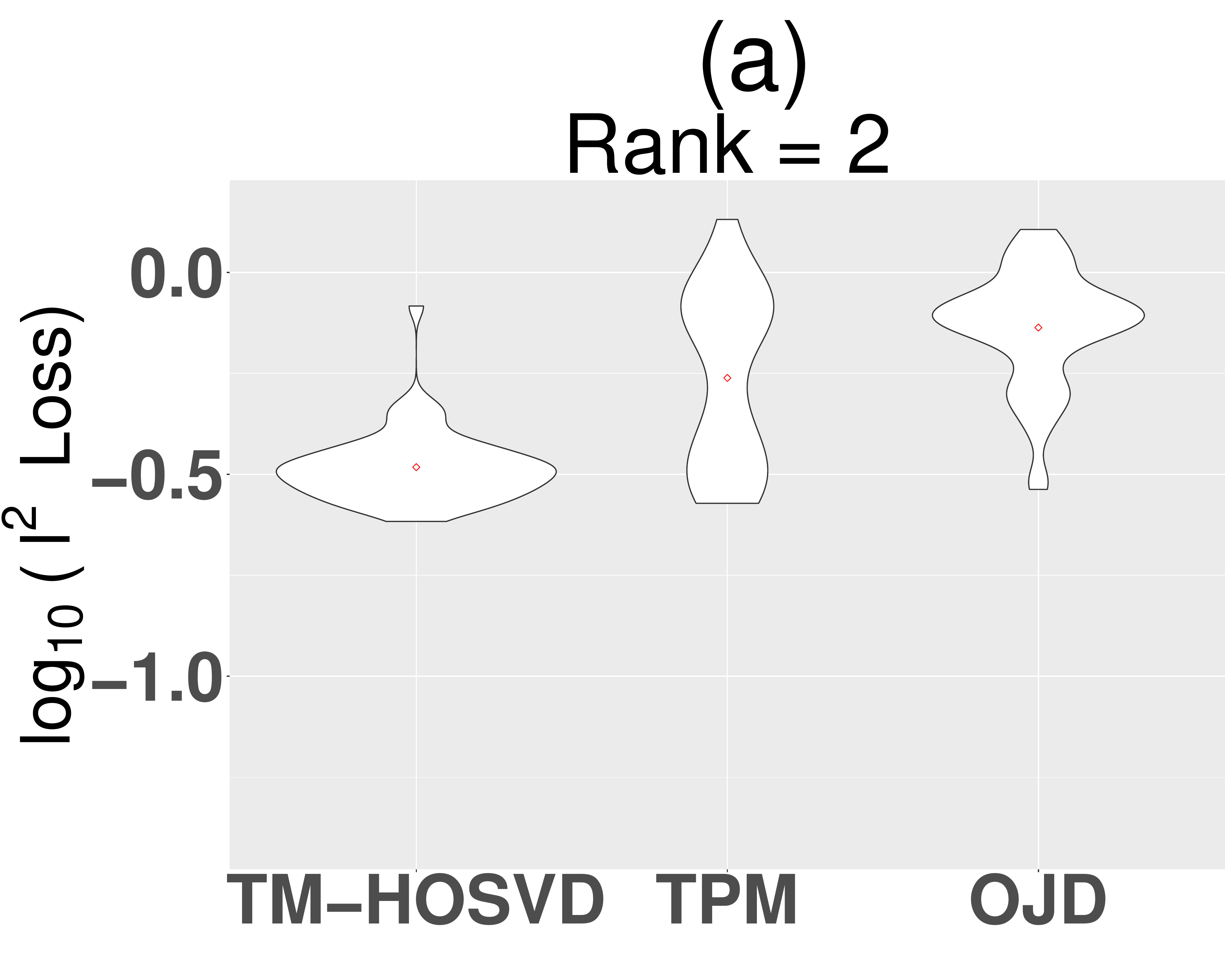}
  \includegraphics[width=4.5cm]{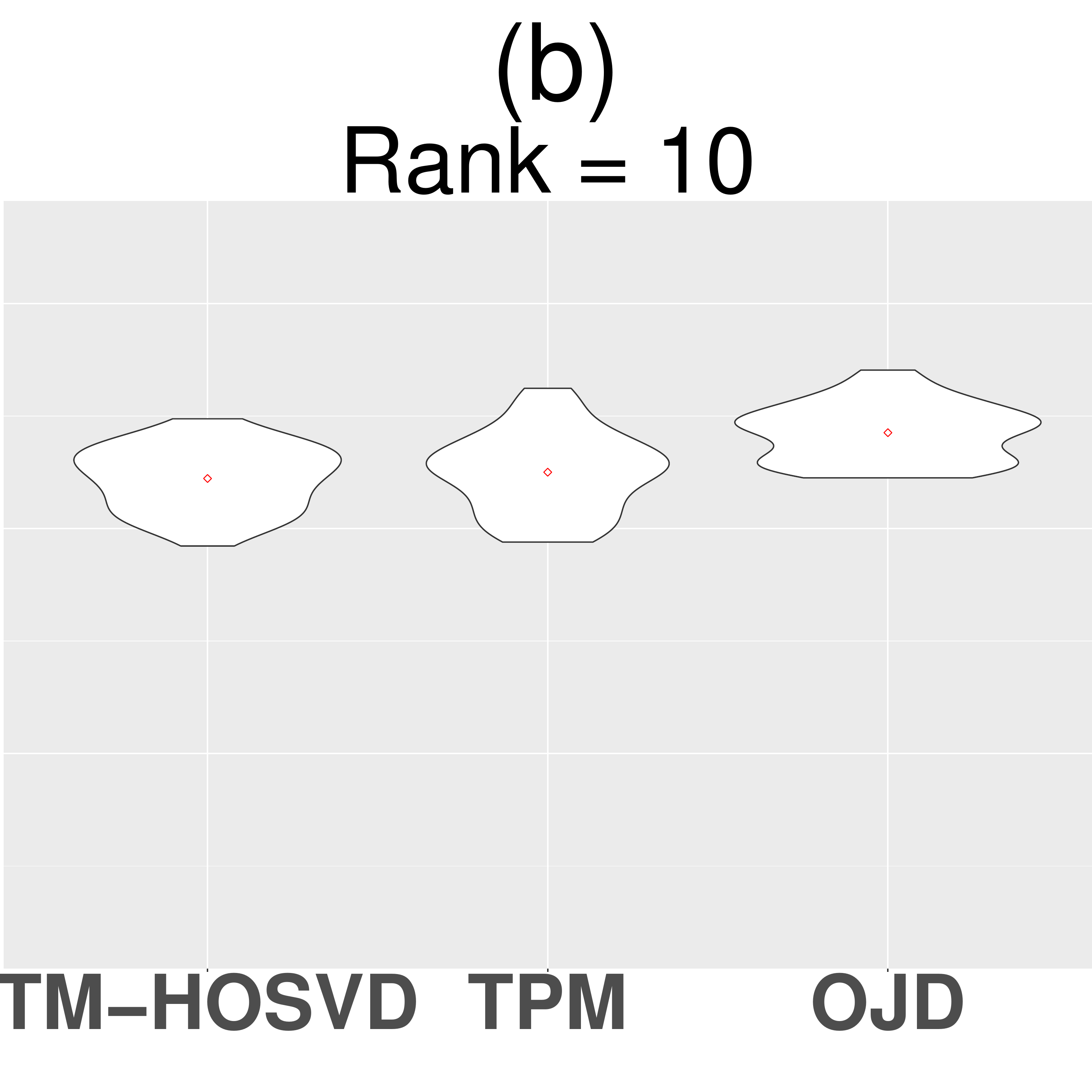}
  \includegraphics[width=4.5cm]{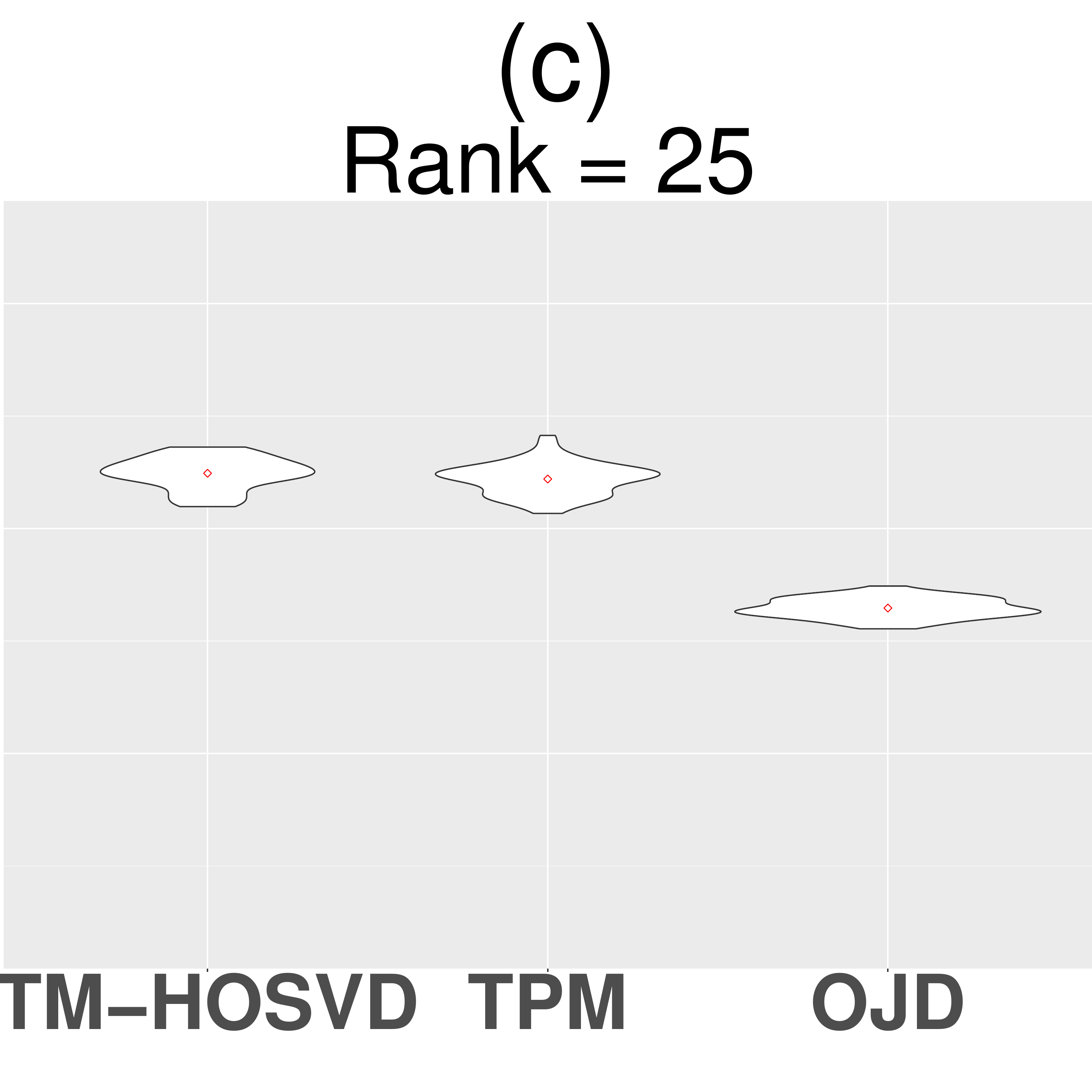}
       \caption{Empirical $l^2$ loss distribution (plotted on a log scale for better visualization) for decomposing order-3 nearly SOD tensors with Gaussian noise, $d=25$ and $\sigma=5\times 10^{-2}$.}\label{fig:distn}
        \end{center}
\end{figure}

Figures~\ref{fig:gaussian} and \ref{fig:distn}, and \suppfref{fig:suppfigure1} (Appendix~\ref{appendix:figures}) demonstrate the robustness of TM-HOSVD to various types of error distributions. For $k=3$, TM-HOSVD outperformed TPM and OJD when the rank is low to moderate. For $k=4$ and $5$, TM-HOSVD consistently achieved a higher accuracy across the full range of the rank (Figure~\ref{fig:order4} and \suppfref{fig:suppfigure2}, Appendix~\ref{appendix:figures}). The empirical runtime is provided in \supptref{tab:supptable1} (Appendix~\ref{appendix:figures}).

We compared the $l^2$ loss distribution of the three methods at the highest noise level ($\sigma=5\times 10^{-2}$ for order-3 tensors and $\sigma=1.5\times 10^{-2}$ for order-4 tensors). The results confirm our earlier conclusion: TM-HOSVD is able to tolerate a greater level of noise. In several scenarios, the loss distribution displayed a bimodal pattern, suggesting a mixture of estimates with good/bad convergence performance. This feature is particularly noticeable in Figure~\ref{fig:distn}a and Figure~\ref{fig:order4}d--f. Further investigation revealed that the poor convergence (e.g.,\ oscillation or local-optimum) occurred less frequently in TM-HOSVD. This is potentially due to the fact that TM-HOSVD starts with the more informative search space $\ls{LS}^{(r)}$ with a lower dimension $r\leq d$, whereas TPM/OJD starts with a random direction in a $d$-dimensional space.

\begin{figure}[t]
   \begin{center}
 \includegraphics[width=15cm]{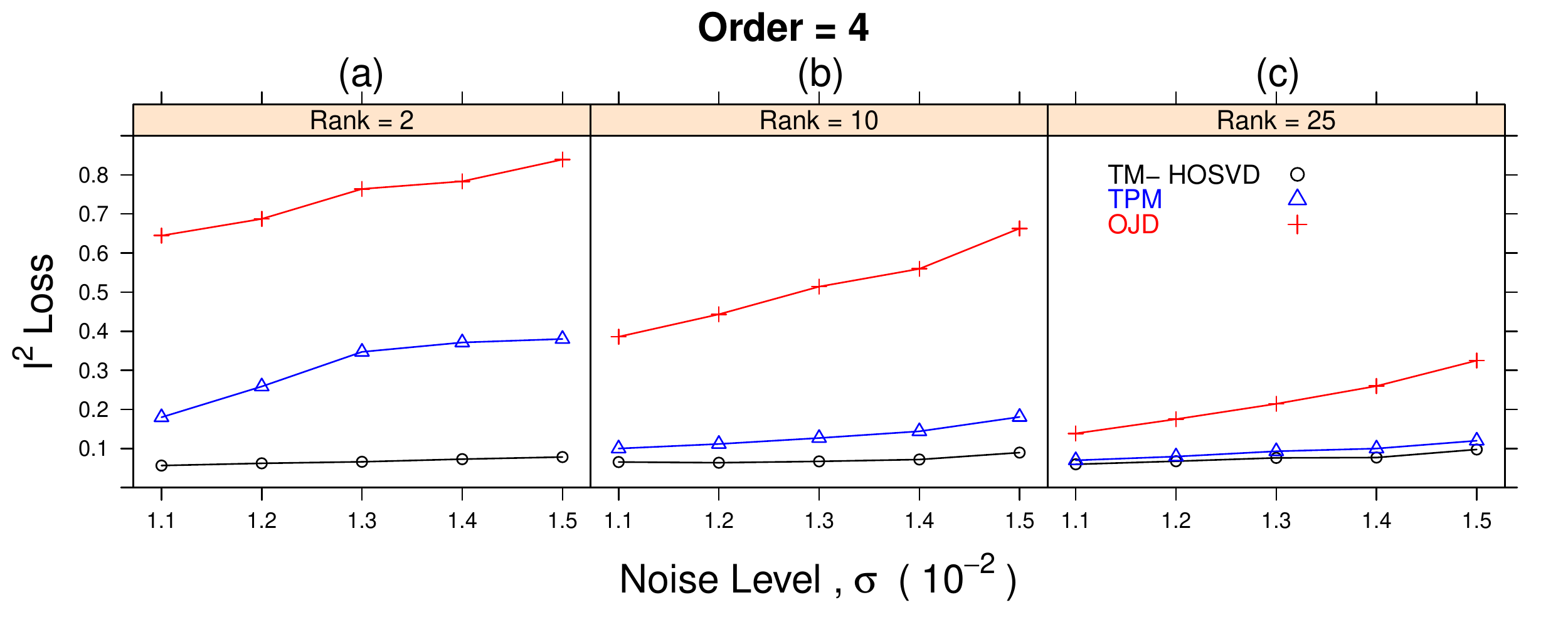}
  \includegraphics[width=5.6cm]{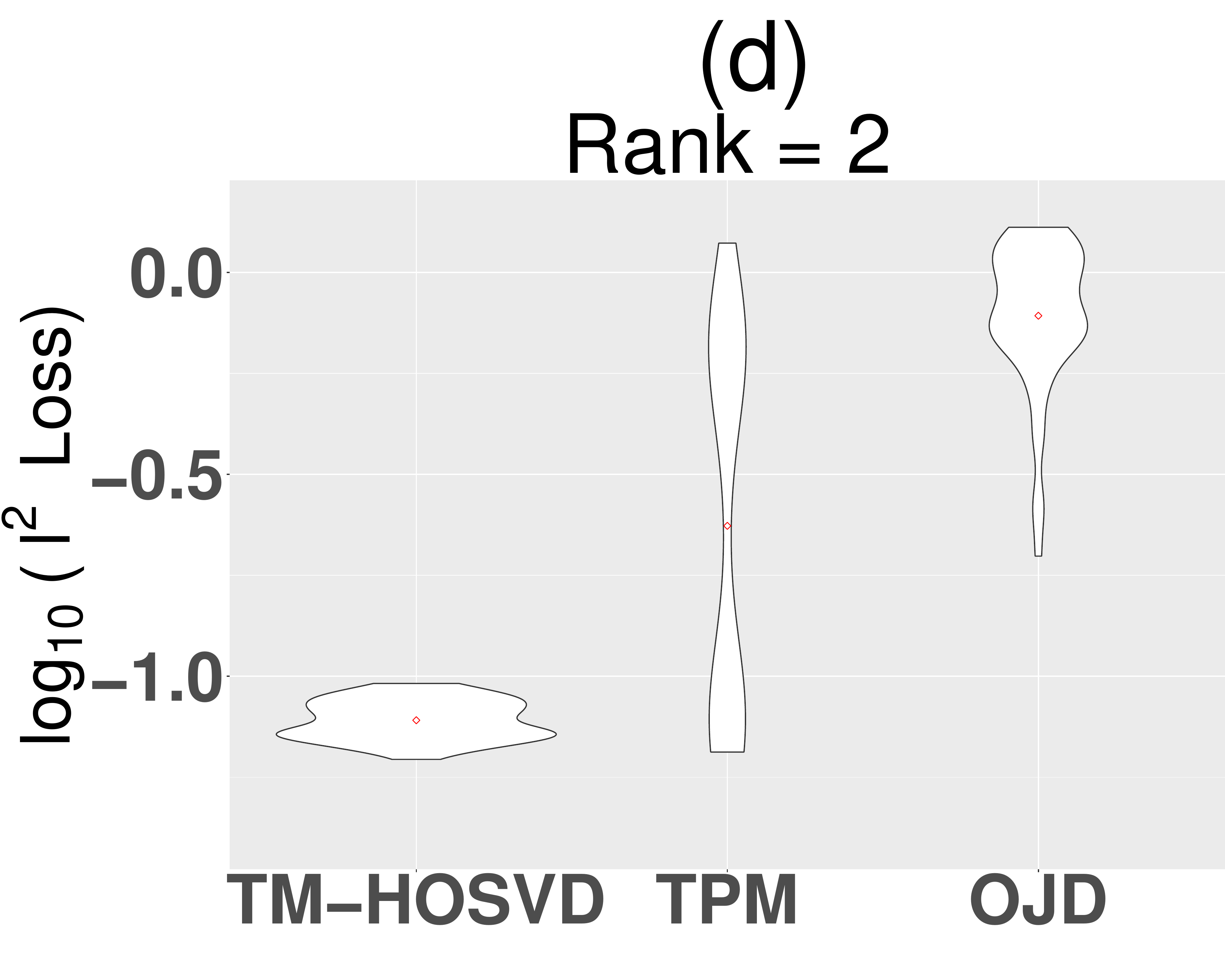}
  \includegraphics[width=4.5cm]{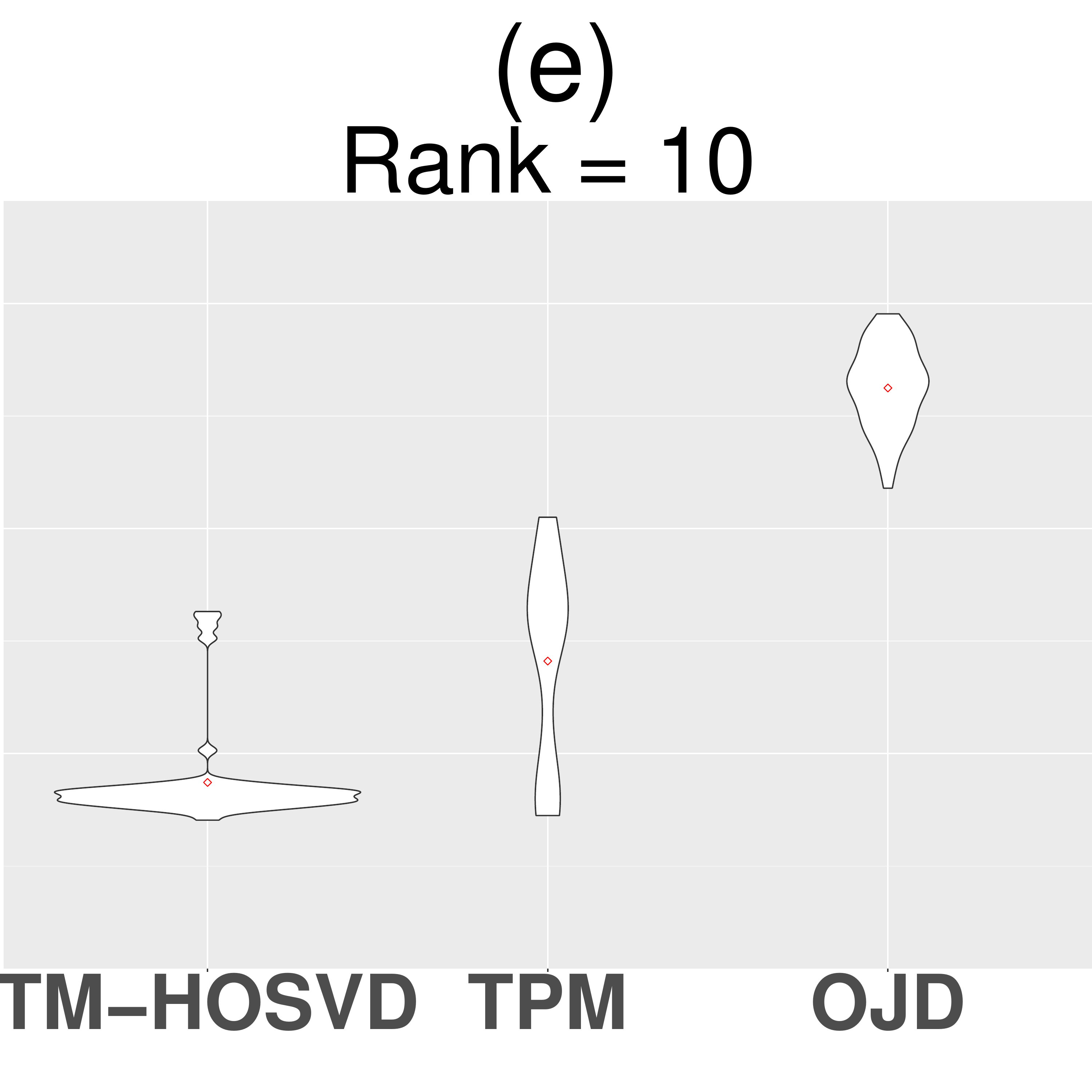}
  \includegraphics[width=4.5cm]{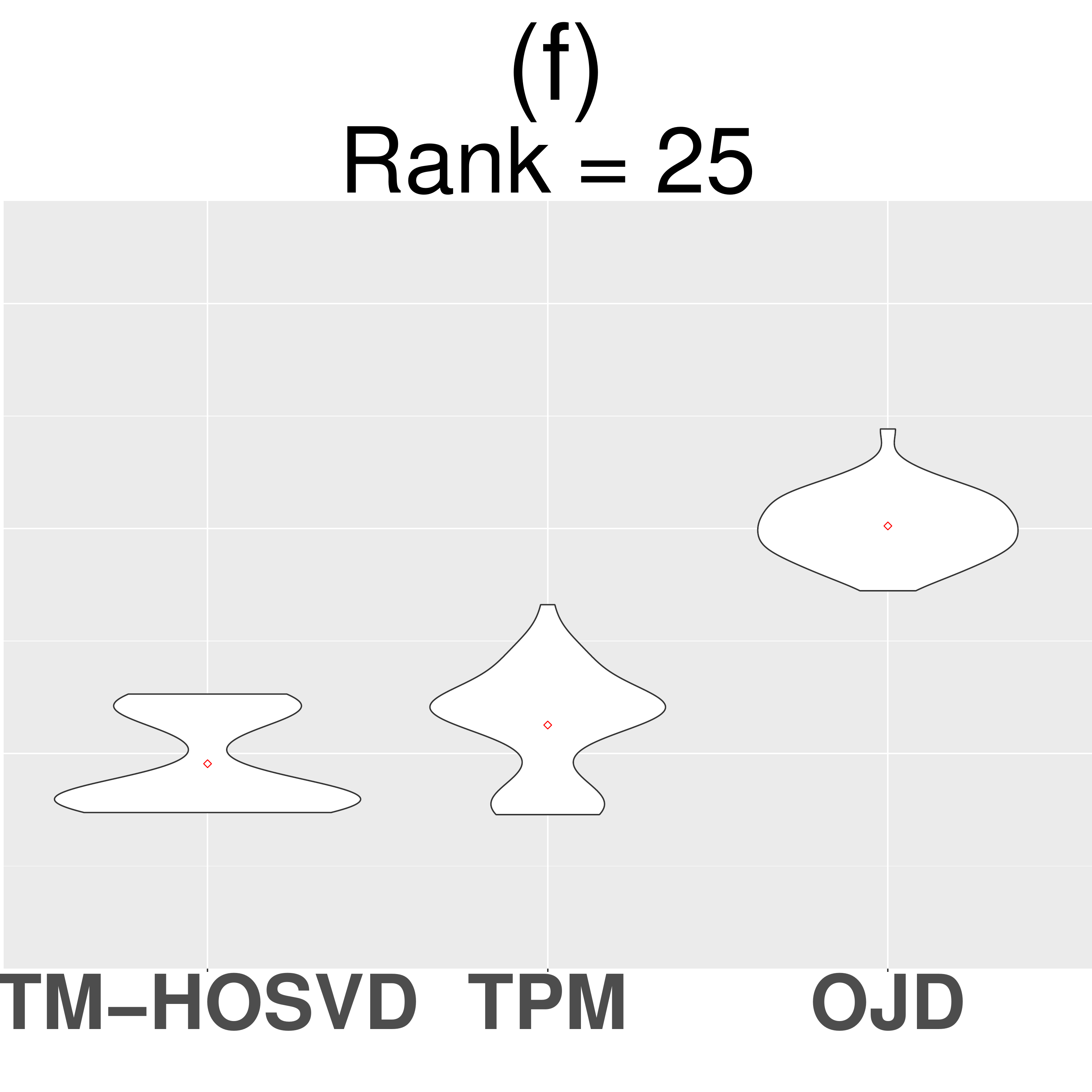}
   \caption{(a--c) Average $l^2$ loss for decomposing order-4 SOD tensors with Gaussian noise, $d=25$. (d--f) Empirical distribution of $l^2$ loss (on a log scale) for decomposing order-4 nearly SOD tensors with Gaussian noise, $d=25$ and $\sigma=1.5\times 10^{-2}$.}\label{fig:order4}
      \end{center}
 \end{figure}

 We also noticed that in the full-rank case for $k=3$, OJD occasionally achieved the best accuracy among all three methods (Figure~\ref{fig:gaussian}c and Figure~\ref{fig:distn}c). We hypothesize that in the full-rank case the best rank-$r$ approximation of tensors may not be achieved by successive rank-1 approximations. Both TPM and TM-HOSVD methods take a greedy approach, so the first several estimated factors tend to explain the most signal. In contrast, OJD optimizes the objective using all $d$ factors simultaneously. Thus, the set of $d$ estimators may explain more signal overall while the first several factors explain less.  This phenomenon also explains why OJD works better for full-rank tensors than for low- and moderate-rank tensors.

\section{Proofs}\label{sec:proofs}
Here, we provide proofs of the theoretical results presented in Sections~\ref{sec:SOD} and \ref{sec:nearlySOD}.

\subsection{Proof of Theorem~\ref{thm:characterization}}\label{appendix:proofcharacterization}

\begin{proof} [Proof of Theorem~\ref{thm:characterization}]
The necessity is obvious. To prove the sufficiency, note that the tensor decomposition $\tT=\sum_{i=1}^r\lambda_i\matr{u}_i^{\otimes k}$ 
implies the two-mode HOSVD:
\begin{equation}\label{eq:matrixSVD}
\tT_{(12)(3\ldots k)}=\sum_{i=1}^r\lambda_i\Vec(\matr{u}_i^{\otimes 2})\Vec(\matr{u}_i^{\otimes (k-2)})^T,
\end{equation}
where each $\lambda_i> 0$ and $\Vec(\matr{u}_i^{\otimes 2})$ is the $i$th left singular vector corresponding to $\lambda_i$.
Now suppose $\Vec(\matr{a}^{\otimes 2})$ is the left singular vector of $\tT_{(12)(3\dots k)}$ corresponding to a non-zero singular value $\lambda\in\bbR\backslash\{0\}$. Then, by \eqref{eq:matrixSVD}, we must have 
\begin{equation}\label{eq:space}
\Vec(\matr{a}^{\otimes 2})\in\Spanspace\{\Vec(\matr{u}_i^{\otimes 2})\colon i\in[r] \text{ for which } \lambda_{i}=\lambda \}.
\end{equation}
Hence, there exist coefficients $\{\alpha_i\}$ such that $\Vec(\matr{a}^{\otimes 2})=\sum\limits_{ i\in[r]\colon  \lambda_i=\lambda }\alpha_i\Vec(\matr{u}_{i}^{\otimes 2})$. In matrix form, this reads 
\begin{equation}\label{eq:rank}
\matr{a}^{\otimes 2}=\sum_{  i\in[r] \colon  \lambda_i=\lambda}\alpha_i \matr{u}_{i}^{\otimes 2},
\end{equation}
where $\{\matr{u}_i\}$ is a set of orthonormal vectors. Notice that the matrix on the right-hand side has rank $|\{ i\in[r]\colon \lambda_i=\lambda\}|$ while the matrix on the left-hand side has rank 1. Since the rank of a matrix is unambiguously determined, we must have $|\{ i\in[r]\colon \lambda_i=\lambda\}|=1$. Therefore, $\matr{a}^{\otimes 2}=\matr{u}_{i^*}^{\otimes 2}$ holds for some $i^*\in[r]$; that is, $\matr{a}$ is a robust eigenvector of $\tT$. 
\end{proof}


\subsection{Proof of Proposition~\ref{prop:searchspace}}\label{appendix:proofsearchspace}

\begin{proof}[Proof of Proposition~\ref{prop:searchspace}]
Suppose $\matr{M}$ is a rank-1 matrix in $\ls{LS}_0=\Spanspace\{\matr{u}_1^{\otimes 2},\ldots,\matr{u}_r^{\otimes 2}\}$, where each $\matr{u}_i$ is a robust eigenvector of $\tT$. Thus, there exist coefficients $\{\alpha_i\}_{i\in[r]}$ such that
\[
\matr{M}=\alpha_1 \matr{u}_1^{\otimes 2}+\cdots+\alpha_r \matr{u}_r^{\otimes 2}. 
\]
Notice that $\{\matr{u}_i\}$ is a set of orthonormal vectors and the rank of a matrix is unambiguously determined. We must have $|\{ i\in[r]\colon \alpha_i\neq 0\}|=1$. Hence, $\matr{M}=\alpha_{i^*}\matr{u}_{i^*}^{\otimes 2}$ holds for some $i^*\in[r]$.
\end{proof}

\subsection{Proof of Theorem~\ref{thm:noiseless}}\label{appendix:proofnoiseless}

\begin{proof}[Proof of Theorem~\ref{thm:noiseless}]
Note that every matrix $\matr{M}\in\ls{LS}_0$ can be written as $\matr{M}=\alpha_1 \matr{u}_1^{\otimes 2}+\cdots+\alpha_r \matr{u}_r^{\otimes 2}$, where $\{\alpha_i\}_{i\in[r]}$ is a set of scalars in $\bbR$. Thus, the optimization problem is equivalent to 
\begin{equation}\label{optimization2}
\max_{\alpha_1^2+\cdots+\alpha^2_r=1} \norm{\alpha_1 \matr{u}_1^{\otimes 2}+\cdots+\alpha_r \matr{u}_r^{\otimes 2}}=\max_{\alpha_1^2+\cdots+\alpha_r^2=1} \max_{i\in[r]}|\alpha_i|.
\end{equation}
Let $f(\boldsymbol{\alpha})=\max_{i\in[r]}\left| \alpha_i \right|$ denote the objective function in \eqref{optimization2}, where $\boldsymbol{\alpha}=(\alpha_1,\ldots, \alpha_r)^T\in\matr{S}^{r-1}$. Notice that the objective is upper bounded by 1; i.e., $f(\boldsymbol{\alpha})\leq 1$ for all $\boldsymbol{\alpha}\in\mathbf{S}^{r-1}$. Suppose $\boldsymbol{\alpha^*}=(\alpha^*_1,\ldots, \alpha^*_r)^T\in\matr{S}^{r-1}$ is a local maximizer of \eqref{optimization2}. We show below that $f(\boldsymbol{\alpha}^*)=1$.

Suppose $f(\boldsymbol{\alpha}^*)\neq 1$. Then we must have $\max_{i\in[r]}|\alpha^*_i|<1$. Without loss of generality, assume $\alpha^*_1$ is the element with the largest magnitude in the set $\{\alpha^*_i\}_{i\in[r]}$. Since $|\alpha^*_1|<1$ and $(\alpha^*_1)^2+\cdots+(\alpha^*_r)^2=1$, there must also exist some $j\geq 2$ such that $\alpha^*_j\neq 0$. Without loss of generality again, assume $\alpha^*_2\neq 0$. 
Now construct another vector $\boldsymbol{\widetilde \alpha}=(\widetilde \alpha_1,\ldots, \widetilde \alpha_r)^T\in\bbR^r$, where
\begin{equation}
\widetilde \alpha_i=
\begin{cases}
\alpha^*_1\eta, &i=1,\\
\text{sign}(\alpha^*_2) \sqrt{(\alpha^*_2)^2-(\eta^2-1)(\alpha^*_1)^2}, &i=2,\\
\alpha^*_i,& i=3,\ldots, r,
\end{cases}
\end{equation}
and $\eta\in \bbR_{+}$ is any value in $\displaystyle \Big(1, { \sqrt{(\alpha^*_1)^2+(\alpha^*_2)^2}\over \alpha^*_1} \Big]$.
It is easy to verify that $\boldsymbol{\widetilde \alpha}\in \matr{S}^{r-1}$ for all such $\eta$. Moreover, 
\begin{equation}\label{eq:distance}
\begin{aligned}
\vectornorm{\boldsymbol{\widetilde \alpha}- \boldsymbol{\alpha}^*}^2 &=\sum_{i=1}^r (\widetilde \alpha_i- \alpha^*_i)^2\notag = (\alpha^*_1)^2(\eta -1)^2+(\alpha^*_2-{\widetilde \alpha}_2)^2 \notag \\
&\leq (\alpha^*_1)^2(\eta -1)^2+(\alpha^*_2)^2+({\widetilde \alpha}_2)^2-2(\widetilde \alpha_2)^2\notag  
=2(\alpha^*_1)^2 \eta (\eta-1).
\end{aligned}
\end{equation}
As we see in the right-hand side of the above inequality, the distance between $\boldsymbol{\widetilde \alpha}$ and $\boldsymbol{\alpha}^*$ can be arbitrarily small as $\eta\rightarrow 1^+$. However, $f(\boldsymbol{\widetilde \alpha})=|\alpha^*_1\eta |> f(\boldsymbol{\alpha}^*)$, which contradicts the local optimality of $ \boldsymbol{\alpha^*}$. Hence, we must have $f(\boldsymbol{\alpha^*})=1$, which completes the proof of (A1). As an aside, we have also proved that every local maximizer of \eqref{optimization2} is a global maximizer.

To see that there are exactly $r$ pairs of maximizers in $\ls{LS}_0$, just notice that $\norm{\matr{M}^*}/\Fnorm{\matr{M}^*}=1$ is equivalent to saying $\matr{M}^*$ is a rank-1 matrix. Thus by Proposition~\ref{prop:searchspace}, $\matr{M}^*=\pm \matr{u}_i^{\otimes 2}$ for some $i\in[r]$.  Conversely, every matrix of the form $\pm \matr{u}_i^{\otimes 2}$ is a maximizer in $\ls{LS}_0$ since $\norm{\matr{u}_i^{\otimes 2}}=1$. The conclusions (A2) and (A3) then follow from the property of $\{\matr{u}_i^{\otimes 2}\}_{i\in[r]}$. 
\end{proof}


\subsection{Auxiliary Theorems}
The following results pertain to standard perturbation theory for the singular value decomposition of matrices. 
For any matrix $\matr{X}$, we use $\matr{X}^\dag$ to denote the Hermitian transpose of $\matr{X}$.
Given a diagonal matrix $\matr{\Sigma}$ of singular values, let $\sigma_{\min}(\matr{\Sigma})$ 
and $\sigma_{\max}(\matr{\Sigma})$  denote, respectively, the minimum and the maximum singular values in $\matr{\Sigma}$.

\begin{thm}[Wedin \cite{wedin1972perturbation}] \label{wedin}
Let $\matr{B}$ and $\nmB$ be two $m\times n$ ($m\geq n$) real or complex matrices with SVDs 
\begin{equation}\label{SVD1}
\matr{B}=\matr{U}\matr{\Sigma}\matr{V}^\dag \equiv \left(\matr{U}_1,\matr{U}_2\right) 
\left(
\begin{array}{cc}
\matr{\Sigma}_1 & \matr{0}\\
\matr{0}&\matr{\Sigma}_2 \\
\matr{0}&\matr{0}\\
\end{array}
\right)
\left(
\begin{array}{c}
\matr{V}_1^\dag\\
\matr{V}_2^\dag\\
\end{array}
\right),
\end{equation}

\begin{equation}\label{SVD2}
\nmB=\matr{\widetilde U}\nmSig\matr{\widetilde V}^\dag \equiv \left(\matr{\widetilde U}_1,\matr{\widetilde U}_2\right) 
\left(
\begin{array}{cc}
\nmSig_1 & \matr{0}\\
\matr{0}&\nmSig_2 \\
\matr{0}&\matr{0}\\
\end{array}
\right)
\left(
\begin{array}{c}
\matr{\widetilde V}_1^\dag\\
\matr{\widetilde V}_2^\dag\\
\end{array}
\right),
\end{equation}
and 
\begin{equation}\label{SVD3}
\begin{aligned}
\matr{\Sigma}_1&=\emph{diag}\left(\sigma_1,\ldots,\sigma_k\right), \quad \matr{\Sigma}_2=\emph{diag}\left(\sigma_{k+1},\ldots,\sigma_n\right),\\
\nmSig_1&=\emph{diag}\left(\widetilde \sigma_1,\ldots,\widetilde \sigma_k\right), \quad \nmSig_2=\emph{diag}\left(\widetilde \sigma_{k+1},\ldots,\widetilde \sigma_n\right),\\
\end{aligned}
\end{equation}
with $\sigma_1\geq \sigma_2\geq \cdots \geq \sigma_n$ and $\tilde\sigma_1\geq \tilde\sigma_2\geq \cdots \geq \tilde\sigma_n$ in descending order.
If there exist an $\alpha\geq 0$ and a $\delta>0$ such that
\begin{equation}\label{eq:gapcondition}
\sigma_{\min}(\matr{\Sigma}_1)=\sigma_k \geq \alpha+\delta \quad \text{and}\quad \sigma_{\max}(\nmSig_2)=\tilde \sigma_{k+1} \leq \alpha,
\end{equation}
then 
\begin{equation}\label{eq:wedinconclusion}
\max\left \{ \norm{\sin \Theta (\matr{U}_1,\matr{\widetilde U}_1)}, \norm{\sin \Theta (\matr{V}_1,\matr{\widetilde V}_1)}\right \}
\leq {\max \left\{\norm{\nmB\matr{V}_1-\matr{U}_1\matr{\Sigma}_1}, \norm{\nmB^{\dag}\matr{U}_1-\matr{V}_1\matr{\Sigma}_1} \right\} \over \delta}.\\
\end{equation}
\end{thm}


\begin{rmk}\label{rm} In the above theorem,
$\matr{U}_1$, $\matr{\widetilde U}_1$ are $d$-by-$k$ matrices and $\Theta (\matr{U}_1,\matr{\widetilde U}_1)$ denotes the matrix of canonical angles between the ranges of $\matr{U}_1$ and $\matr{\widetilde U}_1$.  If we let $\mathcal{L}$ (standing for ``left'' singular vectors) and $\nL$ denote the column spaces of $\matr{U}_1$ and $\matr{\widetilde U}_1$ respectively, then by definition, $\norm{\sin \Theta (\matr{U}_1,\matr{\widetilde U}_1)}\eqdef \norm{\matr{U}_1^T\matr{\widetilde U}_1^{\perp}}=\max_{\matr{x}\in\mathcal{L}, \matr{y}\in\nL}{\matr{x}^T\matr{y} \over \vectornorm{\matr{x}}\vectornorm{\matr{y}}}$. When no confusion arises, we will simply use $\sin\Theta(\mathcal{L},\nL)$ to denote $\norm{\sin \Theta (\matr{U}_1,\matr{\widetilde U}_1)}$.
\end{rmk}


\begin{prop} \label{prop:angle}
Let $\mathcal{L}_1$, $\mathcal{L}_2$ be two subspaces in $\bbR^d$. Then for any vector $\matr{u}_1\in\mathcal{L}_1$,
\[
\sin \Theta\left(\matr{u}_1,\mathcal{L}_2\right) \leq \sin \Theta \left(\mathcal{L}_1,\mathcal{L}_2\right).
\]
\end{prop}

\begin{proof}
The conclusion follows readily from Remark \ref{rm}.
\end{proof}


\begin{thm}[Weyl \cite{weyl1949inequalities}]
Let $\matr{B}$ and $\nmB$ be two matrices with SVDs \eqref{SVD1}, \eqref{SVD2}, and \eqref{SVD3}, Then,
\[
|\widetilde \sigma_i-\sigma_i|\leq \norm{\nmB-\matr{B}} \quad \text{for all } i=1,\ldots, n.
\]
\end{thm}


In our proofs, we often make use of the following corollary based on Wedin's and Weyl's Theorems. 
\begin{cor} \label{corollary1}
Let $\matr{B}$ and $\nmB$ be two matrices with SVDs \eqref{SVD1}, \eqref{SVD2}, and \eqref{SVD3}. Let $\matr{E}\eqdef\nmB-\matr{B}$, and $\mathcal{L}$, $\mathcal{R}$, $\nL$ and $\nR$ be the column spaces of $\matr{U}_1$, $\matr{V}_1$, $\matr{\widetilde U}_1$ and $\matr{\widetilde V}_1$, respectively. Define $\Delta= \min\left\{ \sigma_{\min}(\matr{\Sigma}_1),\sigma_{\min}(\matr{\Sigma}_1)-\sigma_{\max}(\matr{\Sigma}_2)\right\}$. If $\Delta>\norm{\matr{E}}$, then
\begin{equation}\label{eq:WedinWeyl}
\max\left\{\sin\Theta (\mathcal{L}, \nL), \sin\Theta (\mathcal{R}, \nR)\right\} \leq { \norm{\matr{E}} \over \Delta- \norm{\matr{E}}}.
\end{equation}
\end{cor}

\begin{proof}
By Weyl's theorem, $\sigma_{\max}(\matr{\Sigma}_2) -\sigma_{\max}(\nmSig_2)\geq - \norm{\matr{E}}$.
Combining this with the assumption $\sigma_{\min}(\matr{\Sigma}_1)-\sigma_{\max}(\matr{\Sigma}_2)>\norm{\matr{E}}$, we have
\[
\sigma_{\min}(\matr{\Sigma}_1)- \sigma_{\max}(\nmSig_2)=\sigma_{\min}(\matr{\Sigma}_1)- \sigma_{\max}(\matr{\Sigma}_2)+\sigma_{\max}(\matr{\Sigma}_2)- \sigma_{\max}(\nmSig_2) >\norm{\matr{E}}-\norm{\matr{E}}=0.
\]
This implies that the spectrum of $\matr{\Sigma}_1$ is well-separated from that of $\nmSig_2$, and thus \eqref{eq:gapcondition} holds with $\alpha=\max\{0,\sigma_{\max}(\nmSig_2) \}\geq 0$ and $\delta=\sigma_{\min}(\matr{\Sigma}_1)-\alpha>0$. By Wedin's theorem, we get
\[
\max\left\{\sin\Theta (\mathcal{L}, \nL), \sin\Theta (\mathcal{R}, \nR)\right\} \leq {\left\{ \norm{\nmB\matr{V}_1-\matr{U}_1\matr{\Sigma}_1}, \norm{\nmB^{\dag}\matr{U}_1-\matr{V}_1\matr{\Sigma}_1}\right\}\over\delta }.
\]
Then, noting
\[
\norm{\nmB\matr{V}_1-\matr{U}_1\matr{\Sigma}_1}=\norm{\nmB\matr{V}_1-\matr{B}\matr{V}_1}=\norm{\nmB-\matr{B}}=\norm{\matr{E}},
\]
\[
\norm{\nmB^\dag \matr{U}_1-\matr{V}_1\matr{\Sigma}_1}=\norm{\nmB^\dag\matr{U}_1-\matr{B}^\dag\matr{U}_1}=\norm{\nmB^\dag-\matr{B}^\dag}=\norm{\matr{E}},
\]
and
\[
\delta=\sigma_{\min}(\matr{\Sigma}_1)- \max\{0, \sigma_{\max}(\nmSig_2) \}\geq \sigma_{\min}(\matr{\Sigma}_1)- \max\{0, \sigma_{\max}(\matr{\Sigma}_2) \} -\norm{\matr{E}}= \Delta-\norm{\matr{E}},
\]
we obtain \eqref{eq:WedinWeyl}.
\end{proof}


\begin{lem}[Taylor Expansion] 
If $\varepsilon=o(1)$, then
\begin{itemize}
\item $\left(\displaystyle 1+\varepsilon\right)^{\alpha}=1+\alpha\varepsilon+o(\varepsilon), \quad \forall \alpha\in\bbR;$ 
\item $\sin \varepsilon=\varepsilon+o(\varepsilon^2);$
\item $\displaystyle \cos \varepsilon=1-{1\over 2}\varepsilon^2+o(\varepsilon^2).$
\end{itemize}
\end{lem}

\subsection{Proof of Proposition~\ref{prop:uniqueness} (Uniqueness of $\ls{LS}^{(r)}$)}\label{appendix:proofuniqueness}

\begin{proof}[Proof of Proposition~\ref{prop:uniqueness}]
Let $\nHOSVD=\sum_i \mu_i \matr{a}_i\matr{b}_i^T$ be the two-mode HOSVD with $\{\mu_i\}$ in descending order, and $\ls{LS}^{(r)}=\Spanspace\{\matr{a}_1,\ldots,\matr{a}_r\}$ is the $r$-truncated two-mode singular space. In order to show that $\ls{LS}^{(r)}$ is uniquely determined, it suffices to show that $\mu_{r}$ is strictly larger than $\mu_{r+1}$. 

Note that the tensor perturbation model $\ntensor=\sum_{i=1}^r\lambda_i \matr{u}_i^{\otimes k}+\tE$
implies the matrix perturbation model 
\begin{equation}\label{eq:matrixmodel}
\nHOSVD=\sum_{i=1}^r \lambda_i \Vec(\matr{u}_i^{\otimes 2}) \Vec(\matr{u}_i^{\otimes (k-2)})^T+\tE_{(12)(3\ldots k)},
\end{equation}
where by \cite{wang2016operator}
\begin{equation}\label{eq:noise}
\norm{\tE_{(12)(3\ldots k)}}\leq d^{(k-2)/2}\norm{\tE}\leq d^{(k-2)/2}\varepsilon.
\end{equation}

Now apply Corollary~\ref{corollary1} to \eqref{eq:matrixmodel} with $\nmB=\nHOSVD$, $\matr{B}=\sum_{i=1}^r \lambda_i \Vec(\matr{u}_i^{\otimes 2}) \Vec(\matr{u}_i^{\otimes (k-2)})^T$, and $\nmB-\matr{B}=\tE_{(12)(3\ldots k)}$. Considering the corresponding $r$th and $(r+1)$th singular values of $\nmB$ and $\matr{B}$, we obtain 
\begin{equation}\label{eq:weyl}
| \mu_r-\lambda_r|\leq \norm{\tE_{(12)(3\ldots k)}},\quad \text{and} \quad |\mu_{r+1}-0|\leq \norm{\tE_{(12)(3\ldots k)}},
\end{equation}
which implies
\[
 \mu_{r}-\mu_{r+1} = \lambda_r+( \mu_{r}-\lambda_r)-(\mu_{r+1}-0) \geq \lambda_r-2\norm{\tE_{(12)(3\ldots k)}}.
\]
By \eqref{eq:noise} and Assumption~\ref{assumption},
\[
\lambda_r-2\norm{\tE_{(12)(3\ldots k)}}\geq \lambda_{\min}-2d^{(k-2)/2}\varepsilon>0.
\]
Therefore $\mu_r>\mu_{r+1}$, which ensures the uniqueness of $\ls{LS}^{(r)}$.
\end{proof}

\subsection{Proof of Theorem~\ref{thm:LS} (Perturbation of $\ls{LS}_0$)}\label{appendix:proofLS}

\begin{defn}[Singular Space]
Let $\nHOSVD\in \bbR^{d^2\times d^{k-2}}$ be the two-mode unfolding of $\ntensor$, and $\nHOSVD=\sum_i\mu_i \matr{a}_i \matr{b}^T_i$
be the two-mode HOSVD with $\mu_1\geq \mu_2\geq\cdots\geq\mu_r$ in descending order. We define the $r$-truncated left (respectively, right) singular space by  
\[
\begin{aligned}
\ls{LS}^{(r)}&=\Spanspace\left\{\text{Mat}(\matr{a}_i)\in \bbR^{d\times d}\colon \matr{a}_i \text{ is the $i$th left singular vector of }\nHOSVD,  i\in[r]\right\},\\
\ls{RS}^{(r)}&=\Spanspace\left\{\matr{b}_i\in \bbR^{d^{k-2}}\colon \matr{b}_i\text{ is the $i$th right singular vector of }\nHOSVD,  i\in[r]\right\}.
\end{aligned} 
\]
The noise-free version ($\varepsilon=0$) reduces to
\[
\ls{LS}_0=\Spanspace\left\{\matr{u}^{\otimes 2}_i \colon  i\in[r]\right\},\quad \text{and}\quad \ls{RS}_0=\Spanspace\left\{\Vec(\matr{u}^{\otimes (k-2)}_i) \colon  i\in[r]\right\}.
\]
\end{defn}

\begin{rmk}
We make the convention that the elements in $\ls{LS}^{(r)}$ (respectively, $\ls{LS}_0$) are viewed as $d$-by-$d$ matrices, while the elements in $\ls{RS}^{(r)}$ (respectively, $\ls{RS}_{0}$) are viewed as length-$d^{k-2}$ vectors. For  g of notation, we drop the subscript $r$ from $\ls{LS}^{(r)}$ (respectively, $\ls{RS}^{(r)}$) and simply write $\ls{LS}$ (respectively, $\ls{RS}$) hereafter.\\
\end{rmk}

\begin{defn}[Inner-Product]
For any two tensors $\tA=\entry{a_{i_1\dots\, i_k}},$ $\tB=\entry{b_{i_1\dots\, i_k}} \in \bbR^{d_1 \times \cdots \times d_k}$ of identical order and dimensions, their inner product is defined as
\[
\langle \tA,\ \tB \rangle = \sum_{i_1,\dots,i_k} a_{i_1\dots i_k} b_{i_1\dots i_k}, 
\]
while the tensor Frobenius norm of $\tA$ is defined as
\[
\Fnorm{\tA}  = \sqrt{\langle \tA,\ \tA \rangle} = \sqrt{\sum_{i_1, \dots, i_k} |a_{i_1 \dots i_k}|^2}, 
\]
both of which are analogues of standard definitions for vectors and matrices. 
\end{defn}

\begin{lem}\label{lem:LS}
For every matrix $\matr{M}\in\ls{LS}$ satisfying $\Fnorm{\matr{M}}=1$, there exists a unit vector $\matr{b}_{\matr{M}}\in\ls{RS}$ such that
\begin{equation}\label{eq:LS}
\matr{M}=c\ntensor_{(1)(2)(3\ldots k)} (\matr{I},\matr{I},\matr{b}_{\matr{M}}),
\end{equation}
where $c=1/\Fnorm{\ntensor_{(1)(2)(3\ldots k)} (\matr{I},\matr{I},\matr{b}_{\matr{M}})}$ is a normalizing constant. 
\end{lem}

\begin{proof}
Let $\ntensor_{(12)(3\ldots k)}=\sum_i \mu_i \matr{a}_i \matr{b}^T_i$ denote the two-mode HOSVD. Following a similar line of argument as in the proof of Proposition~\ref{prop:uniqueness}, we have $\mu_r\geq \lambda_{\min}-\norm{\tE_{(12)(3\ldots k)}}>0$. By the property of matrix SVD, 
\[
\matr{a}_i={1\over \mu_i} \nHOSVD\matr{b}_i,\quad \text{for all } i\in[r],
\]
which implies
\[
\text{Mat}(\matr{a}_i)= {1\over  \mu_i } \ntensor_{(1)(2)(3\ldots k)}(\matr{I},\matr{I},\matr{b}_i),\quad \text{for all } i\in[r].
\]
Recall that $\ls{LS}=\Spanspace\{\text{Mat}(\matr{a}_i)\colon i\in[r]\}$. Thus, for any $\matr{M}\in \ls{LS}$, there exist coefficients $\{\alpha_i\}_{i\in[r]}$ such that 
\begin{equation}\label{eq:rightsingularspace}
\begin{aligned}
\matr{M}&= \alpha_1 \text{Mat}(\matr{a}_1)+\cdots+\alpha_r \text{Mat}(\matr{a}_r)\\
&={\alpha_1\over \mu_1} \ntensor_{(1)(2)(3\ldots k)}(\matr{I},\matr{I},\matr{b}_1) +\cdots+{\alpha_r\over \mu_r} \ntensor_{(1)(2)(3\ldots k)}(\matr{I},\matr{I},\matr{b}_r)\\
&=\ntensor_{(1)(2)(3\ldots k)} \left(   \matr{I}, \matr{I}, {\alpha_1\over \mu_1}\matr{b}_1+\cdots + {\alpha_r\over \mu_r}\matr{b}_r\right),
\end{aligned}
\end{equation}
where the last line follows from the multilinearity of $\tT_{(1)(2)(3\ldots k)}$. 
Now define $\matr{b}^\prime_{\matr{M}}={\alpha_1\over \mu_1}\matr{b}_1+\cdots+{\alpha_r\over \mu_r}\matr{b}_r$. The conclusion \eqref{eq:LS} then follows by setting $\matr{b}_{\matr{M}}=\matr{b}^\prime_{\matr{M}}/\vectornorm{\matr{b}^\prime_{\matr{M}}}\in\ls{RS}$. 
\end{proof}


\begin{lem}[Perturbation of $\ls{RS}_0$] \label{lem:RS}
Under Assumption~\ref{assumption}, 
\[
\min_{\matr{b}\in\ls{RS},\vectornorm{\matr{b}}=1}\vectornorm{\matr{b}\big|_{\ls{RS}_0}}\geq 1-{d^{k-2}\over 2\lambda^2_{\min}}\varepsilon^2+o(\varepsilon^2).
\]
where $\matr{b}\big|_{\ls{RS}_0}$ denotes the vector projection of $\matr{b}\in\ls{RS}$ onto the space $\ls{RS}_0$.
\end{lem}

\begin{proof}
As seen in the proof of Proposition~\ref{prop:uniqueness}, $\nHOSVD$ can be written as
\begin{equation}\label{eq:RSmodel}
\nHOSVD=\sum_{i=1}^r \lambda_i \Vec(\matr{u}_i^{\otimes 2}) \Vec(\matr{u}_i^{\otimes (k-2)})^T+\tE_{(12)(3\ldots k)},\quad \text{where}\quad \norm{\tE_{(12)(3\ldots k)}}\leq d^{(k-2)/2}\varepsilon.
\end{equation}
The noise-free version of \eqref{eq:RSmodel} reduces to
\begin{equation}\label{eq:RSmodel0}
\tT_{(12)(3\dots k)}=\sum_{i=1}^r \lambda_i \Vec(\matr{u}_i^{\otimes 2}) \Vec(\matr{u}_i^{\otimes (k-2)})^T.
\end{equation}
Following the notation of Corollary~\ref{corollary1}, we set $\nmB=\nHOSVD$, $\matr{B}=\tT_{(12)(3\ldots k)}$, $\matr{\Sigma}_1=\text{diag}\left\{ \lambda_1,\ldots,\lambda_r\right\}$, $\matr{\Sigma}_2=\text{diag}\{ 0,\ldots,0\}$, and $\Delta=\min\{ \sigma_{\min}(\matr{\Sigma}_1),\sigma_{\min}(\matr{\Sigma}_1)-\sigma_{\max}(\matr{\Sigma}_2)\}=\min_{i\in[r]}\lambda_i$. Then, $\norm{\nmB-\matr{B}}=\norm{\tE_{(12)(3\ldots k)}}$. By Assumption~\ref{assumption}, $\Delta=\lambda_{\min}>2d^{(k-2)/2}\varepsilon> \norm{\tE_{(12)(3\ldots k)}}$. Hence the condition of Corollary~\ref{corollary1} holds. Applying Corollary~\ref{corollary1} then yields
\begin{equation}\label{eq:wedin}
\begin{aligned}
\sin\Theta\left(\ls{RS}_0, \ls{RS} \right) 
& \leq {\norm{\tE_{(12)(3\ldots k)}} \over  \lambda_{\min}-\norm{\tE_{(12)(3\ldots k)}}}={\norm{\tE_{(12)(3\ldots k)}}  \over\lambda_{\min}} \left[1-{\norm{\tE_{(12)(3\ldots k)}} \over \lambda_{\min}}\right]^{-1}\\
&\leq {d^{(k-2)/2}\varepsilon \over\lambda_{\min}} \left[1-{d^{(k-2)/2}\varepsilon \over \lambda_{\min}}\right]^{-1} 
={d^{(k-2)/2} \over \lambda_{\min}}\varepsilon +o(\varepsilon).
\end{aligned}
\end{equation}

Now let $\matr{b}\in\ls{RS}$ be a unit vector. Decompose $\matr{b}$ into
\[
\matr{b}=\matr{b}\big|_{\ls{RS}_0}+\matr{b}\big|_{{\ls{RS}_0}^\perp},
\]
where $\matr{b}\big|_{\ls{RS}_0}$ and $\matr{b}\big|_{\ls{RS}_0^\perp}$ are vector projections of $\matr{b}$ onto the spaces $\ls{RS}_0$ and $\ls{RS}_0^{\perp}$, respectively. By \eqref{eq:wedin} and Taylor expansion, 
\[
\vectornorm{\matr{b}\big|_{\ls{RS}_0}}=\cos\Theta\left(\matr{b}, \ls{RS}_0\right)
=\left[1-\sin^2\Theta\left(\matr{b}, \ls{RS}_0 \right) \right]^{1/2} 
\geq 1-{d^{k-2}\over 2\lambda^2_{\min}}\varepsilon^2+o(\varepsilon^2).
\]
Since the above holds for every unit vector $\matr{b}\in\ls{RS}$, we conclude 
\[
\min_{\matr{b}\in\ls{RS},\vectornorm{\matr{b}}=1}\vectornorm{\matr{b}\big|_{\ls{RS}_0}}\geq 1-{d^{k-2}\over 2\lambda^2_{\min}}\varepsilon^2+o(\varepsilon^2).
\]
\end{proof}

\begin{cor}\label{cor:RSvalue}
Under Assumption~\ref{assumption},
\[
\min_{\matr{b}\in\ls{RS},\vectornorm{\matr{b}}=1}\vectornorm{\matr{b}\big|_{\ls{RS}_0}}\geq 1-{1\over (c_0-1)^2 },
\]
which is $\geq 0.98$ for $c_0 \geq 10$.
\end{cor}

\begin{proof} 
Note that ${\norm{\tE_{(12)(3\ldots k)}} \over  \lambda_{\min}}\leq {1\over c_0}$ by Assumption~\ref{assumption}. The right-hand side of \eqref{eq:wedin} can be bounded as follows,  
\[
{\norm{\tE_{(12)(3\ldots k)}} \over  \lambda_{\min}-\norm{\tE_{(12)(3\ldots k)}}}\leq {1\over c_0-1}.
\]
By a similar argument as in the proof of Lemma~\ref{lem:RS}, we obtain
\begin{equation}\label{eq:cos}
\min_{\matr{b}\in\ls{RS},\vectornorm{\matr{b}}=1}\vectornorm{\matr{b}\big|_{\ls{RS}_0}}=\cos\Theta(\matr{b},\ls{RS}_0)\geq \cos^2\Theta(\matr{b},\ls{RS}_0)\geq \cos^2\Theta(\ls{RS},\ls{RS}_0)\geq 1-{1\over (c_0-1)^2},
\end{equation}
which is the desired result.
\end{proof}


\medskip
\begin{proof} [Proof of Theorem~\ref{thm:LS}]
To prove \eqref{eq:pertLS}, it suffices to show that for every matrix $\matr{M}\in\ls{LS}$ satisfying $\Fnorm{\matr{M}}=1$, there exist coefficients $\{\alpha_i\in \bbR\}_{i=1}^r$ such that 
\begin{equation}\label{eq:decomposition}
\matr{M}=\sum_{i=1}^r \alpha_i \matr{u}_i^{\otimes 2}+\matr{E},\quad \text{where }\norm{\matr{E}} \leq {d^{(k-3)/2}\over \lambda_{\min}}\varepsilon+o(\varepsilon).
\end{equation}

Let $\matr{M}$ be a $d$-by-$d$ matrix satisfying $\matr{M}\in\ls{LS}$ and $\Fnorm{\matr{M}}=1$. By Lemma~\ref{lem:LS}, there exists $\matr{b}_{\matr{M}}\in\ls{RS}$ such that
\begin{equation}\label{eq:conclusion}
\begin{aligned}
\matr{M}&={\ntensor_{(1)(2)(3\ldots k)} (\matr{I},\matr{I},\matr{b}_{\matr{M}})\over\Fnorm{\ntensor_{(1)(2)(3\ldots k)} (\matr{I},\matr{I},\matr{b}_{\matr{M}})}}\\
&=\sum_{i=1}^r {\lambda_i \langle \Vec(\matr{u}_i^{\otimes (k-2)}),\matr{b}_{\matr{M}} \rangle \over \Fnorm{\ntensor_{(1)(2)(3\ldots k)} (\matr{I},\matr{I},\matr{b}_{\matr{M}})}}   \matr{u}_i^{\otimes 2}  +{\tE_{(1)(2)(3\ldots k)}(\matr{I},\matr{I},\matr{b}_{\matr{M}})\over  \Fnorm{\ntensor_{(1)(2)(3\ldots k)} (\matr{I},\matr{I},\matr{b}_{\matr{M}})}}.
\end{aligned}
\end{equation}
We now claim that \eqref{eq:conclusion} is a desired decomposition that satisfies \eqref{eq:decomposition}. Namely, we seek to prove
\begin{equation}\label{eq:bound}
{\norm{\tE_{(1)(2)(3\ldots k)}(\matr{I},\matr{I},\matr{b}_{\matr{M}})} \over\Fnorm{\ntensor_{(1)(2)(3\ldots k)} (\matr{I},\matr{I},\matr{b}_{\matr{M}})}}\leq {d^{(k-3)/2} \over \lambda_{\min}} \varepsilon +o(\varepsilon).
\end{equation}
Observe that by the triangle inequality,
\begin{equation}\label{eq:tri}
\begin{aligned}
\Fnorm{\ntensor_{(1)(2)(3\ldots k)}(\matr{I},\matr{I},\matr{b}_{\matr{M}})}&= \Fnorm{\sum_{i=1}^r \lambda_i  \langle \Vec(\matr{u}_i^{\otimes (k-2)}),\matr{b}_{\matr{M}} \rangle \matr{u}_i^{\otimes 2}  +\tE_{(1)(2)(3\ldots k)}(\matr{I},\matr{I},\matr{b}_{\matr{M}})}\\
&\geq   \KeepStyleUnderBrace{ \Fnorm{\sum_{i=1}^r  \lambda_i\langle \Vec(\matr{u}_i^{\otimes (k-2)}),\matr{b}_{\matr{M}} \rangle \matr{u}_i^{\otimes 2}}}_\text{Part I}-\KeepStyleUnderBrace{\Fnorm{\tE_{(1)(2)(3\ldots k)}(\matr{I},\matr{I},\matr{b}_{\matr{M}})}}_\text{Part II}.\\
\end{aligned}
\end{equation}
By the orthogonality of $\{\matr{u}_i\}_{i\in[r]}$, Part I has a lower bound,
\begin{equation}\label{eq:part1}
\begin{aligned}
\Fnorm{\sum_{i=1}^r  \lambda_i\langle \Vec(\matr{u}_i^{\otimes (k-2)}),\matr{b}_{\matr{M}} \rangle \matr{u}_i^{\otimes 2}} 
&\geq \lambda_{\min}\sqrt{\sum_{i=1}^r \langle \Vec(\matr{u}_i^{\otimes (k-2)}),\matr{b}_{\matr{M}} \rangle^2}\\
&= \lambda_{\min}\vectornorm{\matr{b}_{\matr{M}}\big|_{\ls{RS}_0}}.
\end{aligned}
\end{equation}

By the inequality between the Frobenius norm and the spectral norm for matrices, Part II has an upper bound, 
\begin{equation}\label{eq:part2}
\Fnorm{\tE_{(1)(2)(3\ldots k)}(\matr{I},\matr{I},\matr{b}_{\matr{M}})} \leq \sqrt{d}\norm{\tE_{(1)(2)(3\ldots k)}(\matr{I},\matr{I},\matr{b}_{\matr{M}})}
\leq \sqrt{d} \norm{\tE_{(1)(2)(3\ldots k)}} \leq d^{(k-2)/2}\varepsilon,
\end{equation}
where we have used the inequality \cite{wang2016operator} that
\begin{equation}\label{eq:fact}
\norm{\tE_{(1)(2)(3\ldots k)}} \leq d^{(k-3)/2}\norm{\tE}.
\end{equation}
Combining \eqref{eq:tri}, \eqref{eq:part1} and \eqref{eq:part2} gives
\begin{equation}\label{eq:fnorm}
\Fnorm{\ntensor_{(1)(2)(3\ldots k)}(\matr{I},\matr{I},\matr{b}_{\matr{M}})} \geq \lambda_{\min}\left[\vectornorm{\matr{b}_{\matr{M}}\big|_{\ls{RS}_0}}-{d^{(k-2)/2}\varepsilon\over \lambda_{\min}}\right]. 
\end{equation}
By Corollary~\ref{cor:RSvalue} and Assumption~\ref{assumption} with $c_0\geq 10$, $\vectornorm{\matr{b}_{\matr{M}}\big|_{\ls{RS}_0}}-{d^{(k-2)/2}\varepsilon\over \lambda_{\min}}\geq 0.98-0.1>0$. So the right-hand side of \eqref{eq:fnorm} is strictly positive. Taking the reciprocal of  \eqref{eq:fnorm} and combining it with \eqref{eq:fact}, we obtain
\begin{equation}\label{eq:ratio}
\begin{aligned}
{\norm{\tE_{(1)(2)(3\ldots k)}(\matr{I},\matr{I},\matr{b}_{\matr{M}})} \over \Fnorm{\ntensor_{(1)(2)(3\ldots k)}(\matr{I},\matr{I},\matr{b}_{\matr{M}})}}
&\leq {d^{(k-3)/2} \varepsilon \over \lambda_{\min} }\left[\norm{\matr{b}_{\matr{M}}\big|_{\ls{RS}_0}}-{d^{(k-2)/2}\varepsilon\over \lambda_{\min}}\right]^{-1}\\
&\leq {d^{(k-3)/2} \varepsilon\over \lambda_{\min}}\left[1-o(\varepsilon)-{d^{(k-2)/2}\varepsilon\over \lambda_{\min}}\right]^{-1}={d^{(k-3)/2}\over \lambda_{\min}} \varepsilon+o(\varepsilon),
\end{aligned}
\end{equation}
where the second line follows from Lemma~\ref{lem:RS}. This completes the proof of \eqref{eq:bound} and therefore \eqref{eq:decomposition}. Since \eqref{eq:decomposition} holds for every $\matr{M}\in\ls{LS}$ that satisfies $\Fnorm{\matr{M}}=1$, and $\sum_{i=1}^r\alpha_i\matr{u}_i^{\otimes 2}\in\ls{LS}_0$, we immediately have
\[
\max_{\matr{M}\in\ls{LS},\Fnorm{\matr{M}}=1}\; \min_{\matr{M}^*\in\ls{LS}_0}\norm{\matr{M}-\matr{M}^*} \leq {d^{(k-3)/2} \over \lambda_{\min}}\varepsilon+o(\varepsilon).
\]
\end{proof}

\begin{rmk}\label{rmk:decomposition}
In addition to \eqref{eq:decomposition}, $\matr{M}$ can also be decomposed into
\[
\matr{M}=\sum_{i=1}^r\alpha_i\matr{u}_i^{\otimes 2}+\matr{E}',\quad \text{where}\quad \norm{\matr{E}'}\leq {2d^{(k-3)/2}\over\lambda_{\min}}\varepsilon+o(\varepsilon),
\]
where $\matr{E}'$ satisfies 
\[
\langle\matr{E}', \matr{u}_i^{\otimes 2}\rangle=0 \quad\text{for all } i\in[r].
\]
To see this, rewrite \eqref{eq:decomposition} as
\[
\begin{aligned}
\matr{M}=\sum_{i=1}^r \alpha_i\matr{u}_i^{\otimes 2}+\matr{E}&=\sum_{i=1}^r \alpha_i\matr{u}_i^{\otimes 2}+\sum_{i=1}^r\langle \matr{E},\matr{u}_i^{\otimes 2}\rangle
\matr{u}_i^{\otimes 2}+\matr{E}-\sum_{i=1}^r\langle \matr{E},\matr{u}_i^{\otimes 2}\rangle\matr{u}_i^{\otimes 2}\\
&=
\KeepStyleUnderBrace{\sum_{i=1}^r \left(\alpha_i+\langle \matr{E},\matr{u}_i^{\otimes 2}\rangle\right)\matr{u}_i^{\otimes 2}}_{\in \ls{LS}_0}+\KeepStyleUnderBrace{\matr{E}-\sum_{i=1}^r\langle \matr{E},\matr{u}_i^{\otimes 2}\rangle\matr{u}_i^{\otimes 2}}_{=:\matr{E}'}.
\end{aligned}
\]
By construction, $\matr{E}'$ satisfies
\begin{align}
\langle \matr{E}',\matr{u}_i^{\otimes 2}\rangle &=\langle \matr{E}-\sum_{j=1}^r\langle \matr{E},\matr{u}_j^{\otimes 2}\rangle\matr{u}_j^{\otimes 2},\ \matr{u}^{\otimes 2}_i\rangle\\
&=\langle \matr{E},\matr{u}_i^{\otimes 2}\rangle-\sum_{j=1}^r\langle \matr{E},\matr{u}_j^{\otimes 2}\rangle\langle \matr{u}^{\otimes 2}_j,\matr{u}^{\otimes 2}_i\rangle\\
&=\langle \matr{E},\matr{u}_i^{\otimes 2}\rangle-\sum_{j=1}^r\langle \matr{E},\matr{u}_j^{\otimes 2}\rangle\delta_{ij}\\
&=0.
\end{align}
Moreover, 
\begin{align}
\norm{\matr{E}'}&\leq \norm{\matr{E}}+\norm{\sum_{i=1}^r\langle \matr{E},\matr{u}_i^{\otimes 2}\rangle\matr{u}_i^{\otimes 2}}\\
&\leq \norm{\matr{E}}+\max_{i}|\langle \matr{E},\matr{u}_i^{\otimes 2}\rangle|\\
&\leq 2\norm{\matr{E}}\\
&\leq {2d^{(k-3)/2}\over \lambda_{\min}}\varepsilon+o(\varepsilon),
\end{align}
where the first line follows from the triangle inequality and the second lines follows from the orthogonality of $\{\matr{u}_i\}_{i\in[r]}$. 
\end{rmk}

\begin{cor}\label{cor:deviationvalue}
Under Assumption~\ref{assumption},
\[
\max_{\matr{M}\in\ls{LS},\Fnorm{\matr{M}}=1}\min_{\matr{M^*}\in\ls{LS}_0}\norm{\matr{M}-\matr{M^*}}\leq{1.13\over c_0},
\]
which is $\leq {0.12}$ for $c_0\geq 10$.
\end{cor}

\begin{proof}
By Corollary~\ref{cor:RSvalue}, the right-hand side of \eqref{eq:ratio} has the following upper bound,
\[
{d^{(k-3)/2} \varepsilon \over \lambda_{\min} }\left[\norm{\matr{b}_{\matr{M}}\big|_{\ls{RS}_0}}-{d^{(k-2)/2}\varepsilon\over \lambda_{\min}}\right]^{-1}\leq {1\over \sqrt{d}c_0}\left[1-{1\over (c_0-1)^2}-{1\over c_0}\right]^{-1}\leq {1.13\over c_0}\leq 0.12.
\]
The claim then follows from the same argument as in the proof of Theorem~\ref{thm:LS}.
\end{proof}

\begin{cor} \label{cor:coefvalue}
Suppose $c_0\geq 10$ in Assumption~\ref{assumption}. In the notation of \eqref{eq:decomposition}, we have
\[
\max_{i\in[r]}|\alpha_i|\leq 1+{1.13\over c_0}\leq 1.12.
\]
\end{cor}

\begin{proof} By the triangle inequality and Corollary~\ref{cor:deviationvalue},
\[
\max_{i\in[r]}|\alpha_i|\leq \sqrt{\sum_{i=1}^r|\alpha_i|^2}=\Fnorm{\matr{M}-\matr{E}} \leq \Fnorm{\matr{M}}+\Fnorm{\matr{E}} \leq 1+{1.13\over c_0}=1.12.
\]
\end{proof}


\subsection{Proof of Lemma~\ref{lem:max}}

\begin{proof}[Proof of Lemma~\ref{lem:max}]
We prove by construction. Define $\matr{M}_i=\matr{u}_i^{\otimes 2}\in\ls{LS}_0$ for $i\in[r]$, and project $\matr{M}_i$ onto the space $\ls{LS}$,
\begin{equation}\label{projection}
\matr{M}_i=\matr{M}_i\big|_{\ls{LS}}+\matr{M}_i\big|_{\ls{LS}^{\perp}},
\end{equation}
where $\matr{M}_i\big|_{\ls{LS}}$ and $\matr{M}_i\big|_{\ls{LS}^{\perp}}$ denote the projections of $\matr{M}_i\in\ls{LS}_0$ onto the vector space $\ls{LS}$ and $\ls{LS}^{\perp}$,  respectively. We seek to show that the set of matrices $\left\{ \matr{M}_i\big|_{\ls{LS}}\colon i\in[r]\right\}$ satisfies 
\begin{equation}\label{eq:max}
{\norm{\matr{M}_i\big|_{\ls{LS}}}\over \Fnorm{\matr{M}_i\big|_{\ls{LS}}}}\geq 1-{d^{(k-2)/2} \over \lambda_{\min}}\varepsilon+o(\varepsilon), \quad\text{for all}\ i\in[r].
\end{equation}

Applying the subadditivity of spectral norm to \eqref{projection} gives
\begin{equation}\label{B11}
\begin{aligned}
\norm{\matr{M}_i\big|_{\ls{LS}}} &\geq \norm{\matr{M}_i} - \norm{\matr{M}_i\big|_{\ls{LS}^{\perp}}}\\
&\geq 1-  \Fnorm{\matr{M}_i\big|_{\ls{LS}^{\perp}}}
= 1-\sin\Theta(\matr{M}_i,\ls{LS})\Fnorm{\matr{M}_i}\\
& \geq 1-\sin\Theta(\ls{LS}_0,\ls{LS}),
\end{aligned}
\end{equation}
where the second line comes from $\norm{\matr{M}_i}=\Fnorm{\matr{M}_i}=1$, $\norm{\matr{M}_i\big|_{\ls{LS}^{\perp}}}\leq \Fnorm{\matr{M}_i\big|_{\ls{LS}^{\perp}}}$, and the last line comes from Proposition~\ref{prop:angle}. By following the same line of argument in Lemma~\ref{lem:RS}, we have
\begin{equation}\label{B12}
\sin\Theta\left(\ls{LS}_0, \ls{LS} \right)\leq {\norm{\tE_{(12)(3\ldots k)}}\over \lambda_{\min}-\norm{\tE_{(12)(3\ldots k)}}} 
\leq {d^{(k-2)/2} \over \lambda_{\min}}\varepsilon+o(\varepsilon).
\end{equation}
Combining \eqref{B11} and \eqref{B12} leads to 
\begin{equation}\label{rightsingularvector}
\norm{\matr{M}_i\big|_{\ls{LS}}} \geq 1-{d^{(k-2)/2} \over \lambda_{\min}}\varepsilon+o(\varepsilon).
\end{equation}
By construction, $\Fnorm{\matr{M}_i\big|_{\ls{LS}}}\leq \Fnorm{\matr{M}_i}=1$, and therefore \eqref{eq:max} is proved. Note that $\matr{M}_i\big|_{\ls{LS}}\in\ls{LS}$ for all $i\in[r]$. Hence,
\begin{equation}\label{eq:ratiobound}
\max\limits_{\matr{M}\in\ls{LS}} {\norm{\matr{M}} \over \Fnorm{\matr{M}}} \geq {\norm{\matr{M}_i\big|_{\ls{LS}}}\over \Fnorm{\matr{M}_i \big|_{\ls{LS}}}} \geq 1- {d^{(k-2)/2} \over \lambda_{\min}}\varepsilon+o(\varepsilon).
\end{equation}
The conclusion then follows by the equivalence 
\[
\max\limits_{\matr{M}\in\ls{LS}} {\norm{\matr{M}} \over \Fnorm{\matr{M}}} =\max\limits_{\matr{M}\in\ls{LS},\Fnorm{\matr{M}}=1} \norm{\matr{M}}.
\]
\end{proof}

\begin{rmk}
The above proof reveals that there are at least $r$ elements $\matr{M}_i\big|_{\ls{LS}}$ in $\ls{LS}$ that satisfy the right-hand side of \eqref{eq:ratiobound}. These $r$ elements are linearly independent, and in fact, $\{\matr{M}_i\big|_{\ls{LS}}\}_{i\in[r]}$ are approximately orthogonal to each other. To see this, we bound $\cos\Theta(\matr{M}_i\big|_{\ls{LS}},\matr{M}_j\big|_{\ls{LS}})$ for all $i, j\in[r]$, with  $i\neq j$. Recall that $\{\matr{M}\eqdef\matr{u}_i^{\otimes 2}\}_{i\in[r]}$ is a set of mutually orthogonal vectors in $\ls{LS}_0$. Then for all $i\neq j$,
\begin{equation}\label{eq:innerproduct}
\begin{aligned}
0=\langle \mathbf{M}_i, \mathbf{M}_j \rangle&=\langle \mathbf{M}_i\big|_{\ls{LS}}+\mathbf{M}_i\big|_{\ls{LS}^\perp}, \mathbf{M}_j\big|_{\ls{LS}} +\mathbf{M}_j\big|_{\ls{LS}^{\perp}}\rangle\\
&=\langle \mathbf{M}_i\big|_{\ls{LS}}, \mathbf{M}_j\big|_{\ls{LS}}\rangle+\langle \mathbf{M}_i\big|_{\ls{LS}^\perp}, \mathbf{M}_j\big|_{\ls{LS}^\perp}\rangle,
\end{aligned}
\end{equation}
which implies $ \langle \mathbf{M}_i\big|_{\ls{LS}}, \mathbf{M}_j\big|_{\ls{LS}}\rangle=-\langle \mathbf{M}_i\big|_{\ls{LS}^\perp}, \mathbf{M}_j\big|_{\ls{LS}^\perp}\rangle$. Hence,

\begin{align}
\left|\cos \Theta (\mathbf{M}_i\big |_{\ls{LS}},\mathbf{M}_j\big|_{\ls{LS}})\right|&={\left|\langle \mathbf{M}_i\big|_{\ls{LS}}, \mathbf{M}_j\big|_{\ls{LS}}\rangle\right| \over \Fnorm{\mathbf{M}_i\big|_{\ls{LS}}} \Fnorm{\mathbf{M}_j\big|_{\ls{LS}}}}\\
&= {\left|\langle \mathbf{M}_i\big|_{\ls{LS}^\perp}, \mathbf{M}_j\big|_{\ls{LS}^\perp} \rangle\right| \over \Fnorm{\mathbf{M}_i\big|_{\ls{LS}}} \Fnorm{\mathbf{M}_j\big|_{\ls{LS}}}}\\
&\leq {\Fnorm{\mathbf{M}_i\big|_{\ls{LS}^{\perp}}}\over\Fnorm{\mathbf{M}_i\big|_{\ls{LS}}}} \times {\Fnorm{ \mathbf{M}_j\big|_{\ls{LS}^{\perp}}} \over \Fnorm{\mathbf{M}_j\big|_{\ls{LS}}}}\\
&\leq \tan^2\Theta(\ls{LS}_0,\ls{LS}),
\end{align}
where the second line comes from \eqref{eq:innerproduct}, the third line comes from Cauchy-Schwarz inequality and the last line uses the fact that $\matr{M}_i,\matr{M}_j\in\ls{LS}_0$. Following the similar argument as in Corollary~\ref{cor:RSvalue} (in particular, the last inequality in \eqref{eq:cos}), we have $|\sin\Theta(\ls{LS}_0,\ls{LS})|\leq {1\over c_0-1}\leq 0.12$ under the assumption $c_0\geq 10$. Thus,
\[
|\cos \Theta (\mathbf{M}_i\big |_{\ls{LS}},\mathbf{M}_j\big|_{\ls{LS}})|\leq \tan^2\Theta(\ls{LS}_0,\ls{LS})\leq 0.015.
\]
This implies $89.2^{\circ}\leq \Theta(\matr{M}_i\big|_{\ls{LS}},\matr{M}_j\big|_{\ls{LS}})\leq 90.8^{\circ}$; that is, $\left\{\matr{M}_i\big|_{\ls{LS}}\right\}_{i\in[r]}$ are approximately orthogonal to each other. 
\end{rmk}

\begin{cor}\label{cor:maxvalue}
Suppose $c_0\geq 10$ in Assumption~\ref{assumption}. Then
\[
\max\limits_{\matr{M}\in\ls{LS},\Fnorm{\matr{M}}=1}\norm{\matr{M}}\geq 1-{1\over c_0-1}\geq 0.88.
\]
\end{cor}

\begin{proof}
As seen in Corollary~\ref{cor:RSvalue}, 
\[
\sin\Theta(\ls{LS}_0,\ls{LS})\leq {\norm{\tE_{(12)(3\ldots k)}}\over \lambda_{\min}-\norm{\tE_{(12)(3\ldots k)}}} \leq {1\over c_0-1}. 
\]
Combining this with \eqref{B11} and \eqref{B12} gives
\[
\norm{\matr{M}_i\big|_{\ls{LS}}}\geq 1- \sin\Theta(\ls{LS}_0,\ls{LS})\geq 1-{1\over c_0-1} \geq 0.88.
\]
The remaining argument is exactly the same as the above proof of Lemma~\ref{lem:max}. 
\end{proof}

\subsection{Proof of Lemma~\ref{lem:bound}}

\begin{proof}[Proof of Lemma~\ref{lem:bound}]
Because of the symmetry of $\ntensor$ and Lemma~\ref{lem:LS}, $\matr{\widehat M}_1$ must be a symmetric matrix. Now let $\matr{\widehat M}_1=\sum_{i=1}^d \gamma_i\matr{x}_i\matr{x}_i^T$ denote the eigen-decomposition of $\matr{\widehat M}_1$, where $\gamma_i$ is sorted in decreasing order and $\matr{x}_i\in\bbR^d$ is the eigenvector corresponding to $\gamma_i$ for all $i\in[d]$. Without loss of generality, we assume $\gamma_1>0$. 
By construction, $\matr{\widehat M}_1=\argmax_{\matr{M}\in\ls{LS},\Fnorm{\matr{M}}=1}\norm{\matr{M}}$. By Lemma~\ref{lem:max}, 
\[
\gamma_1=\norm{\matr{\widehat M}_1}\geq 1- {d^{(k-2)/2} \over \lambda_{\min}}\varepsilon+o(\varepsilon).
\]
Since $\sum_{i}\gamma_i^2=\Fnorm{\matr{\widehat M}_1}^2=1$, $|\gamma_2|\leq \left(1-\gamma_1^2\right)^{1/2}\leq {\sqrt{2}d^{(k-2)/4}\over \sqrt{\lambda_{\min}}}\sqrt{\varepsilon}+o(\sqrt{\varepsilon})$. Define $\Delta:=\min\{\gamma_1,\gamma_1-\gamma_2\}$. Then,
\begin{equation}\label{eq:gap}
\Delta\geq \gamma_1-|\gamma_2|\geq 1-{\sqrt{2}d^{(k-2)/4}\over \sqrt{\lambda_{\min}}}\sqrt{\varepsilon}+o(\sqrt{\varepsilon}).
\end{equation}
Under the assumption $c_0\geq 10$, $\gamma_1\geq 0.88$ by Corollary~\ref{cor:maxvalue}. Hence, $\Delta \geq \gamma_1-|\gamma_2|\geq 0.88-\sqrt{1-0.88^2}\approx 0.41>0$.

By Theorem~\ref{thm:LS}, there exists $\matr{M}^*=\sum_{i=1}^r\alpha_i\matr{u}_i^{\otimes 2}\in\ls{LS}_0$ such that 
\begin{equation}\label{eq:diff}
\norm{\matr{\widehat M}_1-\matr{M}^*}\leq {d^{(k-3)/2} \over \lambda_{\min}}\varepsilon+o(\varepsilon).
\end{equation}
Without loss of generality, suppose the dominant eigenvector of $\matr{M}^*$ is $\matr{u}_1$. Following the notation of Corollary~\ref{corollary1}, we set $\matr{B}=\matr{\widehat M}_1$, $\nmB=\matr{M}^*$, $\matr{E}= \matr{\widehat M}_1-\matr{M}^*$, $\matr{\Sigma}_1=\{\gamma_1\}$ and $\matr{\Sigma}_2=\text{diag}\{\gamma_2,\ldots,\gamma_d\}$. 
From Corollary~\ref{cor:deviationvalue}, $\norm{\matr{E}} \leq 0.12$. Combining this with earlier calculation, we have $\Delta-\norm{\matr{E}}\geq 0.41-0.12=0.29>0$. Hence, the condition in Corollary~\ref{corollary1} holds. 

Applying Corollary~\ref{corollary1} to the specified setting yields
\begin{equation}\label{eq:sintop}
|\sin\Theta(\est,\matr{u}_1)| \leq {\norm{\matr{E}}\over\Delta- \norm{\matr{E}}}
\leq {d^{(k-3)/2}\over \lambda_{\min}}\varepsilon \left[ 1-{\sqrt{2}d^{(k-2)/4}\over \sqrt{\lambda_{\min}}}\sqrt{\varepsilon}+o(\sqrt{\varepsilon}) \right]^{-1}={d^{(k-3)/2}\over \lambda_{\min}}\varepsilon +o(\varepsilon).
\end{equation}
To bound $\Loss(\est,\matr{u}_1)$, we notice that
\[
\Loss(\est,\matr{u}_1)=\left[2-2\left|\cos\Theta (\est, \matr{u}_1)\right|\right]^{1/2}=\left[2-2\sqrt{1-\sin^2\Theta (\est, \true)}\right]^{1/2}.
\]
By Taylor expansion and \eqref{eq:sintop}, we conclude
\[
\Loss(\est,\matr{u}_1)\leq {d^{(k-3)/2}\over \lambda_{\min}}\varepsilon+o(\varepsilon).
\]
\end{proof}

\begin{cor} \label{cor:lossvalue}
Under Assumption~\ref{assumption},
\[
\Loss
(\est,\matr{u}_1) \leq  {5\over c_0},
\]
which is $\leq 0.5$ for $c_0\geq 10$.
\end{cor}

\begin{proof}
In the proof of Lemma~\ref{lem:bound}, we have shown that $\Delta-\norm{\matr{E}}\geq 0.29$. By Corollary~\ref{cor:deviationvalue}, $\norm{\matr{E}}\leq 1.13/c_0$. Therefore, \eqref{eq:sintop} has the following upper bound,
\[
|\sin\Theta(\est,\matr{u}_1)| \leq {\norm{\matr{E}}\over\Delta- \norm{\matr{E}}}\leq {4\over c_0}.
\]
Following the same argument as in the proof of Lemma~\ref{lem:bound}, we obtain
\[
\Loss(\est,\matr{u}_1)=\left[2-2|\cos\Theta(\est,\true)|\right]^{1/2}\leq {5\over c_0}\leq 0.5.
\]
\end{proof}

\subsection{Proof of Lemma~\ref{lem:post-processing}}

\begin{proof}[Proof of Lemma~\ref{lem:post-processing}]
For clarity, we use $\matr{\widehat M}_1$ and $\est$ to denote the estimators in line 5 of Algorithm~\ref{alg:main}, and use $\matr{\widehat M}^*_1$ and $\matr{\widehat u}^*_1$ to denote the estimators in line 6 of Algorithm~\ref{alg:main}. Namely,
\[
\matr{\widehat M}^*_1=\ntensor(\matr{I},\matr{I},\est,\ldots,\est),\quad \text{and}\quad \matr{\widehat u}^*_1=\argmax_{\matr{x}\in\matr{S}^{d-1}}|\matr{x}^T\matr{\widehat M}^*_1\matr{x}|.
\]
By construction, the perturbation model of $\ntensor$ implies the perturbation model of $\matr{\widehat M^*}_1$,
\begin{equation}\label{eq:newdecomposition}
\matr{\widehat M}^*_1=\sum_{i=1}^r \lambda_i \langle \est,\matr{u}_i \rangle^{(k-2)} \matr{u}_i^{\otimes 2}+\tE(\matr{I},\matr{I},\est,\ldots,\est),
\end{equation}
where $\norm{\tE(\matr{I},\matr{I},\est,\ldots,\est)} \leq \norm{\tE}\leq \varepsilon$. 

Without loss of generality, assume $\est$ is the estimator of $\matr{u}_1$ and $\langle \est,\matr{u}_1 \rangle>0$; otherwise, we take $-\est$ to be the estimator. Let $\eta_i:=\lambda_i \langle \est,\matr{u}_i \rangle^{(k-2)}$ for all $i\in[r]$. In the context of Corollary~\ref{corollary1}, we set $\matr{B}=\sum_{i\in[r]}\eta_i\matr{u}_i^{\otimes 2}$, $\nmB=\matr{\widehat M}^*$, $\matr{E}=\nmB-\matr{B}$, $\matr{\Sigma}_1=\{\eta_1\}$, $\matr{\Sigma}_2=\text{diag}\{ \eta_2,\ldots,\eta_r\}$, and $\Delta=\min\{\eta_1,\eta_1-\max_{i\neq 1}\eta_i\}$. Then, 
\begin{equation}\label{eq:eigengap}
\Delta\geq \eta_1-\max_{i\neq 1}|\eta_i|= \lambda_1\langle \est,\matr{u}_1\rangle^{(k-2)} -\max_{i\neq 1}|\lambda_i \langle \est,\matr{u}_{1}\rangle|^{(k-2)}.
\end{equation}
Note that $\norm{\matr{E}}\leq \norm{\tE}\leq \varepsilon$. In order to apply Corollary~\ref{corollary1}, we seek to show $\Delta>\varepsilon$. 

By Definition~\ref{def:loss}, we have
\begin{equation}\label{eq:1}
\langle \est,\matr{u}_1\rangle = \cos \Theta (\matr{\widehat  u}_1,\matr{u}_1)=1-{1\over 2}\Loss^2(\est,\matr{u}_{1}),
\end{equation}
and by the orthogonality of $\{\matr{u}_i\}_{i\in[r]}$, 
\begin{equation}\label{eq:2}
\left| \langle \est, \matr{u}_i\rangle\right|^2 \leq \sum_{j=2}^r \left| \langle \matr{\widehat  u}_1, \matr{u}_j\rangle\right|^2 
\leq 1-\cos^2\Theta(\matr{\widehat  u}_{1},\matr{u}_1)
= \Loss^2(\matr{\widehat u}_{1},\matr{u}_{1})\left[1-{1\over 4}\Loss^2(\matr{\widehat u}_{1},\matr{u}_{1})\right],
\end{equation}
for all $i=2,\ldots,r$.

Combining  \eqref{eq:1}, \eqref{eq:2}, $0\leq \Loss(\est,\matr{u}_1)\leq 1/2$ (by Corollary~\ref{cor:lossvalue}), and the fact that $(1-x)^{(k-2)}\geq 1-(k-2) x$ for all $0\leq x\leq 1$ and $k\geq 3$, we further have
\begin{equation}\label{eq:top}
\langle \matr{\widehat u}_{1},\matr{u}_1\rangle^{(k-2)} = \left[1-{1\over 2}\Loss^2(\matr{\widehat u}_{1},\matr{u}_{1})\right]^{(k-2)}\geq 1-{k-2 \over 2}\Loss^2(\est,\matr{u}_1)\geq 1-{k-2 \over 4}\Loss(\est,\matr{u}_1),
\end{equation}
and 
\begin{equation}\label{eq:others}
|\langle \matr{\widehat  u}_1,\matr{u}_i\rangle|^{(k-2)} \leq \left[\Loss^2(\est,\matr{u}_1)\right]^{(k-2)/2}
=\Loss^{k-2}(\est,\matr{u}_1)\leq\Loss(\est,\matr{u}_1),
\end{equation}
for all $i=2,\ldots,r$.
Putting \eqref{eq:top} and \eqref{eq:others} back in \eqref{eq:eigengap}, we obtain
\begin{align}
\Delta&\geq \lambda_1\left [1-{k-2\over 4}\Loss(\est,\matr{u}_1)\right ]-\lambda_{\max}\Loss(\est,\matr{u}_1)\\
&\geq \lambda_1\left[ 1-\left({k-2\over 4}+{\lambda_{\max} \over \lambda_{\min}}\right)\Loss(\est,\true) \right].
\end{align}
By Corollary~\ref{cor:lossvalue}, $\Loss(\est,\true)\leq 5/c_0$. Write $c:={k-2\over 4}+{\lambda_{\max}\over \lambda_{\min}}$. Under the assumption
$
c_0\geq \max\left\{10, 6c \right\},
$
we have $\Delta\geq \lambda_1/6$ and hence
\[
\Delta-\varepsilon  \geq { \lambda_1\over 6}-{\lambda_{\min}\over c_0 d^{(k-2)/2}}> { \lambda_1\over 6}-{\lambda_{\min}\over 10}> 0.
\]
This implies that the condition in Corollary~\ref{corollary1} holds. Now applying Corollary~\ref{corollary1} to the specified setting gives
\begin{equation}\label{eq:sin}
\begin{aligned}
\left|\sin \Theta (\est, \matr{u}_1)\right|
&\leq {\varepsilon \over \Delta-\varepsilon }\\
&\leq {\varepsilon \over \lambda_1}\left[ 1- c\Loss(\est,\true)-{\varepsilon\over \lambda_1}  \right]^{-1}\\
&\leq {\varepsilon \over \lambda_1}\left[ 1- {cd^{(k-3)/2}\varepsilon\over \lambda_{\min}}-{\varepsilon\over \lambda_1} +o(\varepsilon) \right]^{-1}={\varepsilon\over \lambda_1}+o(\varepsilon),
\end{aligned}
\end{equation}
where the third line follows from Lemma~\ref{lem:bound}.
Using the fact that $\Loss(\est,\true)=\left[2-2|\cos\Theta (\est, \true)|\right]^{1/2}=\left[2-2\sqrt{1-\sin^2\Theta (\est, \true)}\right]^{1/2}$ and Taylor expansion, we conclude
\[
\Loss(\est,\true)\leq {\varepsilon\over \lambda_1}+o(\varepsilon).
\]

To obtain $\Loss(\estv,\truev)$, recall that under the assumption $\langle \est,\true \rangle>0$, $\Loss(\estv,\truev)=|\estv-\truev|$. (Otherwise, we need to consider $|\estv+\truev|$ instead). Observe that by the triangle inequality,
\begin{equation}\label{eigenvalue}
\begin{aligned}
|\estv -\truev|&= \left|\mathcal{T}(\est,\ldots,\est)- \lambda_{1}\right| =\left|\sum_{i=1}^r\lambda_i\langle \est ,\matr{u}_i \rangle^{k}+\tE(\est,\ldots,\est)- \lambda_1\right|\\
&\leq \lambda_1\left|1- \langle\est, \matr{u}_1 \rangle^k \right|+\sum_{i=2}^r \lambda_i \left|\langle \est,\matr{u}_i \rangle\right|^k+\left|\tE(\est,\ldots,\est)\right|.\\
\end{aligned}
\end{equation}
Using similar techniques as in \eqref{eq:1}, \eqref{eq:2}, \eqref{eq:top} and \eqref{eq:others}, as well as the fact $(1-x)^k\geq 1-kx$ for all $0\leq x\leq 1$ and $k\geq 3$, we conclude
\begin{equation}
\begin{aligned}
|\estv-\lambda_1|&\leq {\lambda_1 k\over 2}\Loss^2(\est,\matr{u}_1) + \lambda_{\max}\Loss^2(\est,\matr{u}_1) +\varepsilon\\
&\leq \left({\lambda_1k\over 2}+\lambda_{\max} \right)\left[{\varepsilon\over \lambda_1}+o(\varepsilon)\right]^2+\varepsilon\\
&=\varepsilon+o(\varepsilon).
\end{aligned}
\end{equation}
\end{proof}

\subsection{Proof of Lemma~\ref{lem:deflation}}

\begin{proof}[Proof of Lemma~\ref{lem:deflation}]
Let $\matr{M}$ be a $d$-by-$d$ matrix in the space $\ls{LS}(X)\eqdef\ls{LS} \cap \Spanspace\{ \matr{\widehat u}_i^{\otimes 2}: i\in X\}^{\perp}$ and suppose $\matr{M}$ satisfies $\Fnorm{\matr{M}}=1$. Since $\ls{LS}(X)\subset \ls{LS}$, from Remark~\ref{rmk:decomposition}, $\matr{M}$ can be decomposed into
\begin{equation}\label{decomplemma}
\matr{M}=\sum_{i=1}^r\alpha_i\matr{u}_i^{\otimes 2}+\matr{E},
\end{equation}
where
\begin{equation}\label{eq:relaxnorm}
\langle \matr{E},\matr{u}_i^{\otimes 2}\rangle=0\quad \text{for all } i\in[r],\quad \text{and}\quad  \norm{\matr{E}}\leq {2d^{(k-3)/2}\varepsilon \over\lambda_{\min}}+o(\varepsilon).
\end{equation}

By definition, every element in $\ls{LS}(X)$ is orthogonal to $\Vec(\esti^{\otimes 2})$ for all $i\in X$. We claim that under this condition, one must have $\alpha_i=o(\varepsilon)$ for all $i\in X$. To show this, we project $\esti$ onto the space $\Spanspace\{\matr{u}_i\}$ and write
\begin{equation}
\esti=\xi_i\matr{u}_i+\eta_i\matr{u}_i^{\perp},
\end{equation}
where $\xi_i^2+\eta_i^2=1$ and $\matr{u}_i^{\perp}\in\mathbf{S}^{d-1}$ denotes the normalized (i.e., unit) vector projection of $\esti$ onto the space $\Spanspace\{\matr{u}_i\}^{\perp}$. Then for all $i \in X$,
\begin{equation}\label{eq:compare}
\begin{aligned}
0&=\langle \matr{M},\esti^{\otimes 2}\rangle \\
&=\Bigg\langle \sum_{j\in[r]}\alpha_j\matr{u}_j^{\otimes 2}+\matr{E},\; \left( \xi_i\matr{u}_i+\eta_i\matr{u}_i^\perp\right)^{\otimes 2}\Bigg\rangle\\
&=\Bigg\langle \alpha_i\matr{u}_i^{\otimes 2}+\sum_{j\neq i,\;  j\in[r]} \alpha_j\matr{u}_j^{\otimes 2}+\matr{E},\ \xi_i^2\matr{u}_i^{\otimes 2}+2\xi_i\eta_i\matr{u}_i\otimes \matr{u}_i^\perp+\eta_i^2(\matr{u}_i^\perp)^{\otimes 2}\Bigg\rangle\\
&=\alpha_i\xi_i^2+2\xi_i\eta_i\Bigg \langle \matr{E},\matr{u}_i\otimes \matr{u}_i^\perp\Bigg \rangle+\eta_i^2\Bigg\langle \sum_{j\neq i,\; j\in[r]}\alpha_j\matr{u}_j^{\otimes 2}+\matr{E},\; (\matr{u}_i^\perp)^{\otimes 2}\Bigg\rangle,\\
\end{aligned}
\end{equation}
where the last line uses the fact that $\langle \matr{E},\matr{u}_i^{\otimes 2}\rangle=0$, $\langle \matr{u}_i, \matr{u}_i^\perp\rangle=0$ and $\langle \matr{u}_i, \matr{u}_j\rangle=0$ for all $j\neq i$.  By assumption, $\Loss(\esti,\matr{u}_i)\leq 2\varepsilon/ \lambda_i+o(\varepsilon)$. This implies $|\eta_i|=|\langle \esti, \matr{u}_i^\perp \rangle| = [1-\cos^2\Theta(\esti,\matr{u}_i)]^{1/2}\leq \Loss(\esti,\matr{u}_i)[1-{1\over 4}\Loss^2(\esti,\matr{u}_i)]^{1/2}\leq \Loss(\esti,\matr{u}_i)=O(\varepsilon)$, and $|\xi_i|=(1-\eta^2_i)^{1/2}\geq 1-O(\varepsilon)$. It then follows from \eqref{eq:compare} that
\begin{align}
\xi_i^2|\alpha_i|&=\bigg|2\xi_i\eta_i\bigg \langle \matr{E},\matr{u}_i\otimes \matr{u}_i^\perp\bigg \rangle+\eta_i^2\Bigg\langle \sum_{j\neq i,\; j\in [r]}\alpha_j\matr{u}_j^{\otimes 2}+\matr{E},\ (\matr{u}_i^\perp)^{\otimes 2}\Bigg\rangle\bigg|\\
&\leq 2|\xi_i\eta_i| \bigg|\bigg \langle \matr{E},\matr{u}_i\otimes \matr{u}_i^\perp \bigg \rangle \bigg|+\eta^2_i  \left(\sum_{j\neq i,\; j\in[r]} \bigg|\alpha_j\bigg \langle \matr{u}_j^{\otimes 2},\ (\matr{u}_i^\perp)^{\otimes 2} \bigg \rangle\bigg| +\bigg |\bigg\langle \matr{E},\ (\matr{u}_i^\perp)^{\otimes 2}\bigg \rangle\bigg|\right) \\
&\leq 2|\xi_i\eta_i|\norm{\matr{E}}+\eta_i^2\left(\sum_{j\neq i,\; j\in[r]}|\alpha_j|+\norm{\matr{E}}\right)\\
&\leq O(\varepsilon)\left({2d^{(k-3)/2}\over \lambda_{\min}}\varepsilon+o(\varepsilon)\right)+O(\varepsilon^2)\left(1.12r+{2d^{(k-3)/2}\over \lambda_{\min}}\varepsilon+o(\varepsilon) \right)=o(\varepsilon),
\end{align}
where the last line follows from $|\eta_i|\leq O(\varepsilon)$, $|\xi_i|\leq 1$, $\norm{\matr{E}}\leq{2d^{(k-3)/2}\varepsilon\over \lambda_{\min}}+o(\varepsilon)$  (c.f., \eqref{eq:relaxnorm}) and $\max_{i\in[r]}|\alpha_i|\leq 1.12$ (c.f., Corollary~\ref{cor:coefvalue}). Therefore, since $|\xi_i|\geq 1-O(\varepsilon)$, we conclude that $|\alpha_i|=o(\varepsilon)$ for all $i\in X$. 

Now write \eqref{decomplemma} as
\[
\matr{M}=\sum_{i\in [r]\backslash X}\alpha_i\matr{u}_i^{\otimes 2}+\sum_{i\in X}\alpha_i\matr{u}_i^{\otimes 2}+\matr{E},
\]
Note that $\sum_{i\in[r]\backslash X} \alpha_i\matr{u}_i^{\otimes 2}\in\ls{LS}_0(X)\eqdef\Spanspace\{\matr{u}_i^{\otimes 2}\colon i \in [r]\backslash X\}$. Hence, 
\[
\min_{\matr{M}^*\in \ls{LS}_0(X)}\norm{\matr{M}-\matr{M}^*}\leq \norm{\sum_{i\in X}\alpha_i\matr{u}_i^{\otimes 2}+\matr{E}}\leq \max_{i\in X}|\alpha_i|+\norm{\matr{E}}\leq {2d^{(k-3)/2}\varepsilon\over\lambda_{\min}}+o(\varepsilon).
\]
Since the above holds for all $\matr{M}\in \ls{LS}(X)$ that satisfies $\Fnorm{\matr{M}}=1$, taking maximum over $\matr{M}$ yields the desired result. 
\end{proof}

\subsection{Proof of Theorem~\ref{thm:main}}
We use the following lemma \cite{mu2015successive} in our proof of Theorem~\ref{thm:main}.

\begin{lem}\label{lem:residual}
Fix a subset $X\subset [r]$ and assume that $0\leq \varepsilon\leq \lambda_i/2$ for each $i\in X$. Choose any $\{\esti,\estvi\}_{i\in X}\subset \bbR^d\times \bbR$ such that
\[
|\lambda_i-\estvi|\leq \varepsilon,\quad \vectornorm{\esti}=1, \quad \text{and}\quad \langle \matr{u}_i,  \esti \rangle\geq 1-2(\varepsilon/\lambda_i)^2>0,
\]
and define tensor $\Delta_i:=\lambda_i\matr{u}_i^{\otimes k}-\estvi\esti^{\otimes k}$ for $i\in X$. Pick any unit vector $\matr{a}=\sum_{i=1}^da_i\matr{u}_i$. Then, there exist positive constants $C_1, C_2>0$, depending only on $k$, such that
\begin{equation}\label{eq:residualconclusion}
\norm{\sum_{i\in X}\Delta_i\matr{a}^{\otimes k-1}}\leq C_1\left( \sum_{i\in X}|a_i|^{k-1}\varepsilon\right)+C_2\left(|X|\left(\varepsilon\over \lambda_{\min} \right)^{k-1} \right),
\end{equation}
where $\Delta_i\matr{a}^{\otimes k-1}:= \Delta_i (\matr{a},\ldots,\matr{a},\matr{I})\in \bbR^d$.
\end{lem}


\medskip
\begin{proof}[Proof of Theorem~\ref{thm:main}]
We prove \eqref{eq:finalconclusion} by induction on $i$. For $i=1$, the error bound of $\{(\matr{\widehat u}_1, {\widehat \lambda}_1)\in\bbR^d\times \bbR\}$ follows readily from Lemmas~\ref{lem:max}--\ref{lem:post-processing}. Now suppose \eqref{eq:finalconclusion} holds for $i\leq s$. Taking $X=[s]$ in Lemma~\ref{lem:deflation} yields the deviation of $\ls{LS}(X)$ from $\ls{LS}_0(X)$,
\begin{equation}\label{eq:iterative}
\max_{\matr{M}\in \ls{LS}(X), \Fnorm{\matr{M}}=1} \; \min_{\matr{M}^*\in\ls{LS}_0(X)} \norm{\matr{M}-\matr{M}^*}\leq {2d^{(k-3)/2}\varepsilon\over \lambda_{\min}}+o(\varepsilon).
\end{equation}
Applying Theorem~\ref{thm:LS} and Lemmas~\ref{lem:max}--\ref{lem:post-processing} to $i=s+1$ with $\varepsilon$ replaced by $2\varepsilon$ in \eqref{eq:maxrank}, \eqref{eq:boundmaintext} and \eqref{eq:firstbound} (because of the additional factor ``2'' in \eqref{eq:iterative} compared to \eqref{eq:pertLS}), we obtain
\[
\Loss(\matr{\widehat u}_{s+1},\matr{u}_{\pi(s+1)})\leq {2\varepsilon\over \lambda_{\pi(s+1)}}+o(\varepsilon),\quad \Loss(\widehat \lambda_{s+1},\lambda_{\pi(s+1)})\leq 2\varepsilon+o(\varepsilon).
\]
So \eqref{eq:finalconclusion} also holds for $i=s+1$. 

It remains to bound the residual tensor $\Delta\ntensor\eqdef\ntensor-\sum_{i\in [r]}\estvi\esti^{\otimes k}$. Note that $\Loss(\esti,\matr{u}_{\pi(i)})\leq 2\varepsilon/ \lambda_{\pi(i)}+o(\varepsilon)$ implies $ \langle \esti,\matr{u}_{\pi(i)}\rangle=1-{1\over 2}\Loss^2(\esti,\matr{u}_{\pi(i)})\geq 1-2(\varepsilon/\lambda_{\pi(i)})^2+o(\varepsilon^2)$. When $c_0$ is sufficiently large (i.e., $\varepsilon$ is sufficiently small), $\esti$ is approximately parallel to $\matr{u}_{\pi(i)}$ and orthogonal to $\matr{u}_j$ for all $j\neq \pi(i)$. For ease of notation, we renumber the indices and assume $\pi(i)=i$ for all $i\in[r]$. Following the definition of $\Delta_i$ in Lemma~\ref{lem:residual}, 
\begin{align}
\norm{\Delta\ntensor}=\norm{\sum_{i\in [r]}\lambda_i\matr{u}_i^{\otimes k}+\tE-\sum_{i\in [r]}\estvi\esti^{\otimes k}}= \norm{\sum_{i\in [r]}\Delta_i+\tE}.
\end{align}
Now taking $X=[r]$ in \eqref{eq:residualconclusion} gives
\begin{equation}
\begin{aligned}
\norm{\Delta\ntensor} &\leq\max_{\matr{a}\in\mathbf{S}^{d-1}}\norm{\sum_{i\in [r]}\Delta_i \matr{a}^{\otimes (k-1)}}+\varepsilon\\
 &\leq  \max_{\matr{a}\in\mathbf{S}^{d-1}}C_1\sum_{i\in [r]}|a_i|^{k-1} \varepsilon+C_2r\left({\varepsilon \over \lambda_{\min}}\right)^{k-1}+\varepsilon\\
&\leq \max_{\matr{a}\in\mathbf{S}^{d-1}} C_1\varepsilon \sum_{i\in[r]}|a_i|^2+C_2r\left({\varepsilon \over \lambda_{\min}}\right)^2+\varepsilon\\
&\leq C\varepsilon+o(\varepsilon),
\end{aligned}
\end{equation}
where the third line comes from the fact that $k\geq 3$, $|a_i|\leq 1$, and $\varepsilon/ \lambda_{\min}\leq 1$ from Assumption~\ref{assumption}.
\end{proof}

\section{Conclusion}

We have proposed a new method for tensor decompositions based on the two-mode HOSVD. This method tolerates a higher level of noise and achieves accuracy bounds comparable to that of existing methods while displaying empirically favorable performance. Our approach extends naturally to asymmetric tensors, e.g., by using the matrix SVD in place of the eigendecomposition in lines 5--6 of Algorithm~\ref{alg:main}. In addition, recent works have shown that some non-orthogonal tensors can be converted to orthogonal tensors by an additional whitening step \cite{anandkumar2014tensor, kolda2015symmetric}. Therefore,  the two-mode HOSVD is applicable to a broad class of structured tensors.  In particular, our proposed algorithm shows stable convergence and exhibits pronounced advantage especially as the order of the tensor increases.

\section*{Acknowledgments}
We thank Khanh Dao Duc and Jonathan Fischer for useful discussion.
This research is supported in part by a Math+X Simons Research Grant from the Simons Foundation and a Packard Fellowship for Science and Engineering.

\clearpage
\appendix
\section*{Appendix}

\section{Supplementary Figures and Table}\label{appendix:figures}

\renewcommand{\thefigure}{{S\arabic{figure}}}%
\renewcommand{\thetable}{{S\arabic{table}}}%
\renewcommand{\figurename}{{Supplementary Figure}}    
\renewcommand{\tablename}{{Supplementary Table}}    
\setcounter{figure}{0}   
\setcounter{table}{0}

\begin{figure}[H]
\begin{center}
\includegraphics[width=15cm]{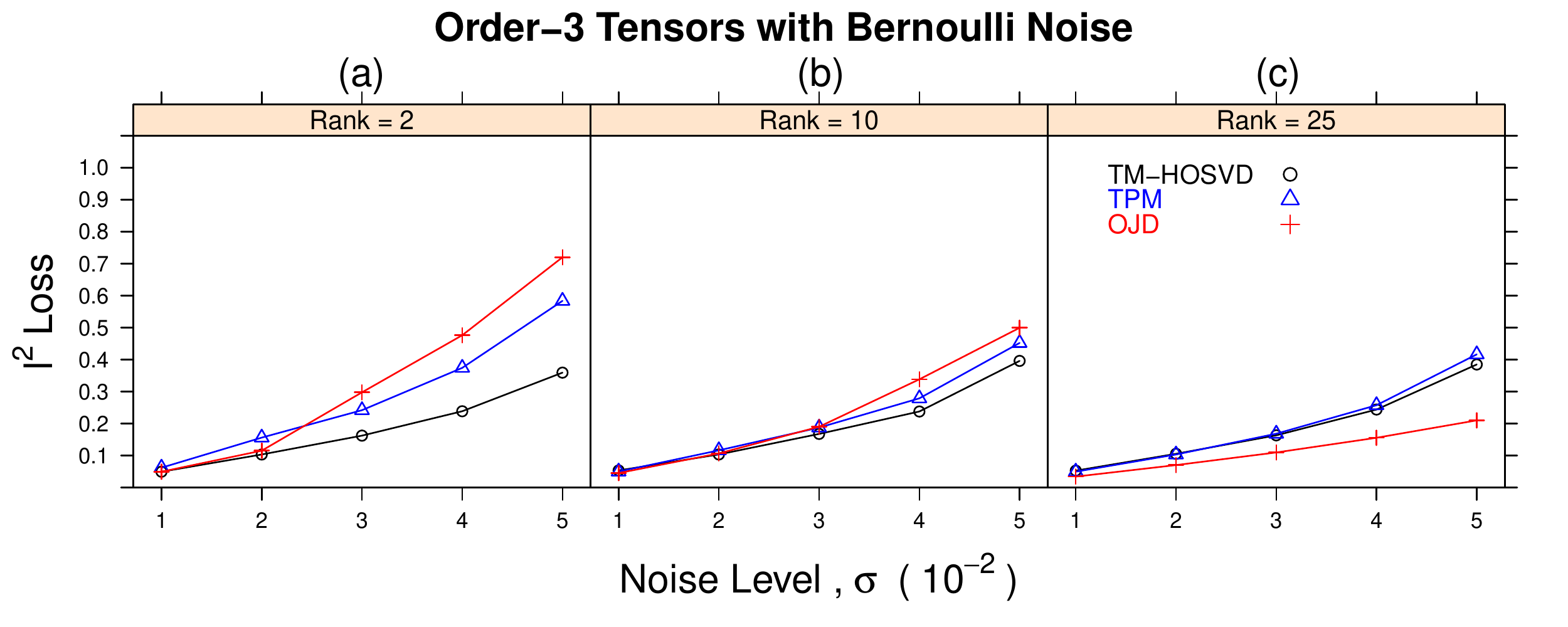}\\
\includegraphics[width=15cm]{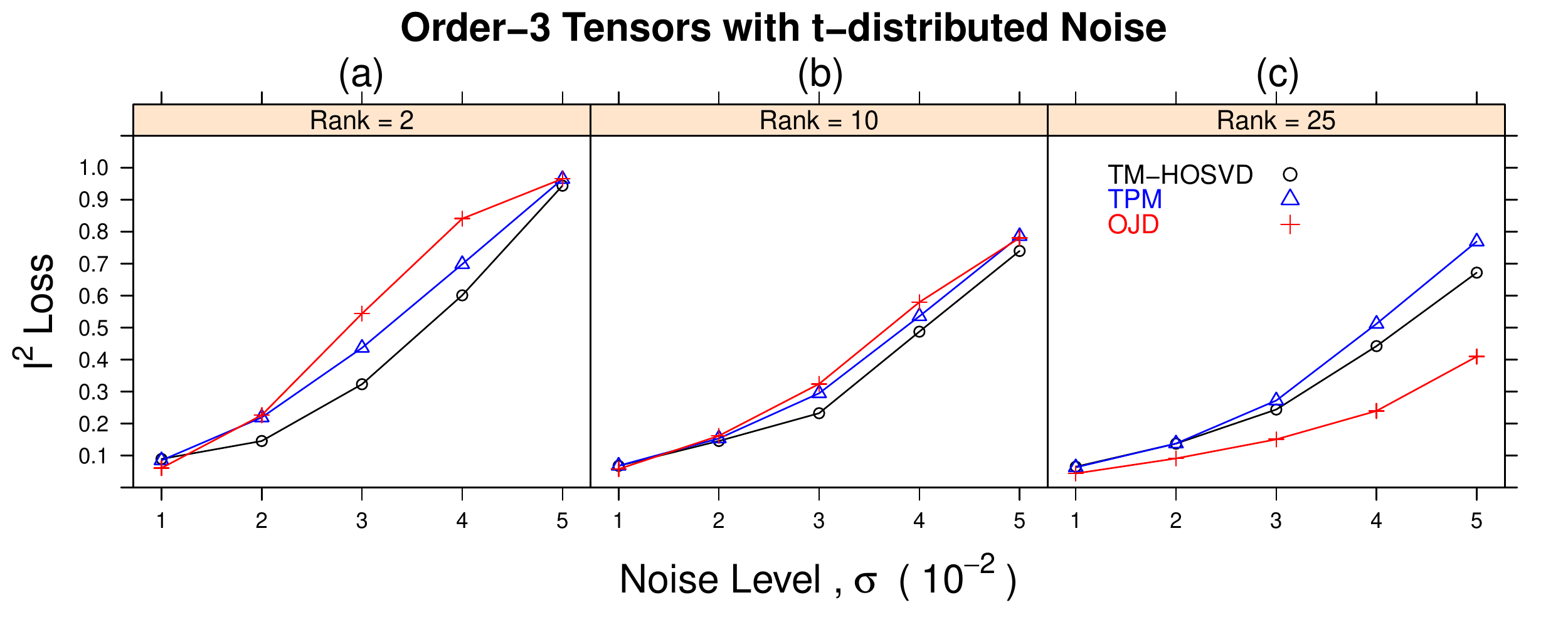}
\end{center}
\caption{Average $l^2$ Loss for decomposing order-3 nearly SOD tensors with Bernoulli/T-distributed noise, $d=25$.}\label{fig:suppfigure1}
\end{figure}

\begin{figure}[H]
\centerline{\includegraphics[width=15cm]{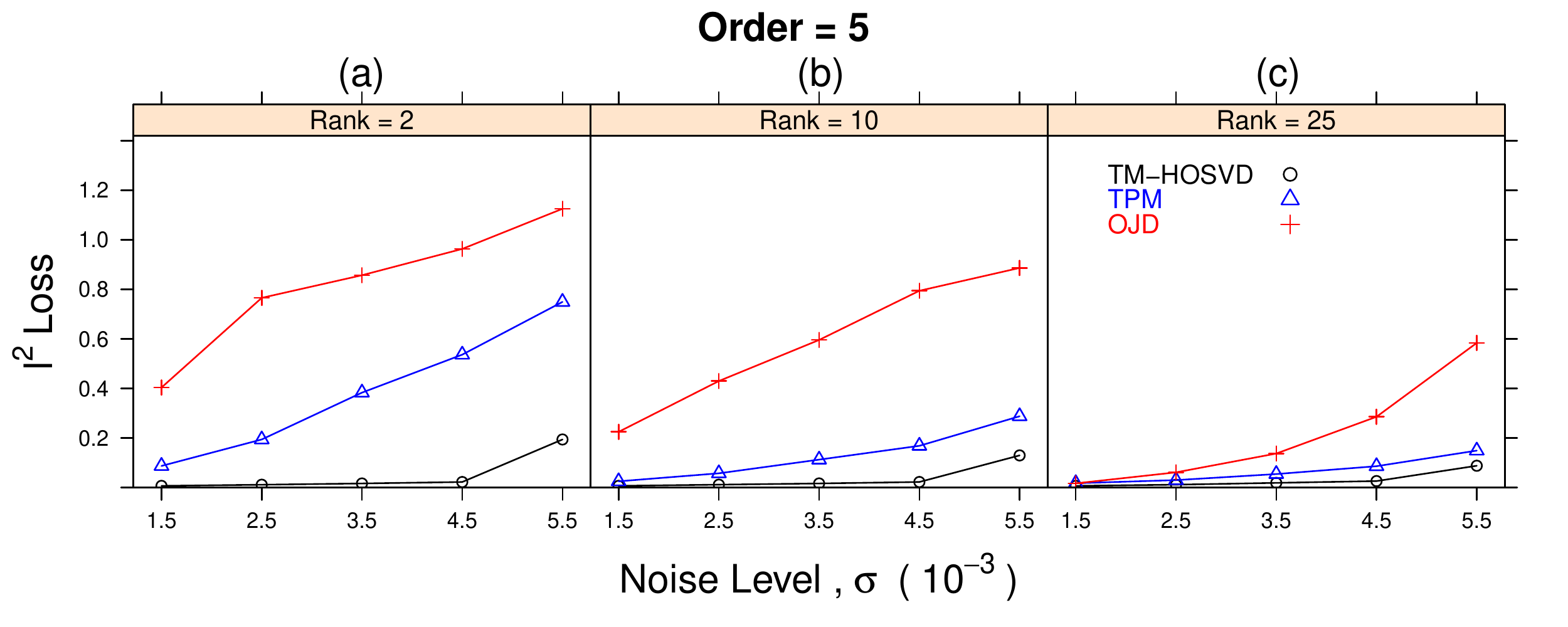}}
  \caption{Average $l^2$ Loss for decomposing order-5 nearly SOD tensors with Gaussian noise, $d=25$.}\label{fig:suppfigure2}
\end{figure}

\begin{table}[H]
\caption{Runtime for decomposing nearly-SOD tensors with Gaussian noise, $d=25$.}\label{tab:supptable1}
\hspace{5mm}
\begin{center}
\begin{tabular}{c|c|c|c|c|c}
\toprule
\multirow{2}{*}{Order}&\multirow{2}{*}{Rank}&\multirow{2}{*}{Noise Level ($\sigma$)}&\multicolumn{3}{c}{Time (sec.)}\\
\cline{4-6}
&&&TM-HOSVD &TPM&OJD\\
\hline
3&2&$5\times 10^{-2}$&0.08&0.01&0.13\\
3&10&$5\times 10^{-2}$&0.20&0.03&0.80\\
3&25&$5\times 10^{-2}$&0.47& 0.07&0.92\\
4&2&$1.5\times 10^{-2}$&0.13&0.06&0.12\\
4&10&$1.5\times 10^{-2}$&0.29 &0.14&1.06\\
4&25&$1.5\times 10^{-2}$&0.57&0.25&1.58\\
5&2&$5.5\times 10^{-3}$&0.25&0.51&0.14\\
5&10&$5.5\times 10^{-3}$& 0.45&1.98&1.01\\
5&25&$5.5\times 10^{-3}$&0.87&4.27&2.66\\
\bottomrule
\end{tabular}
\end{center}
\end{table}

\clearpage
\bibliographystyle{acm}
\bibliography{TensorDecomposition} 

\end{document}